%% file: main.tex
\documentclass{article}
\usepackage{latexml}

\PassOptionsToPackage{numbers,sort&compress}{natbib}
\usepackage[final, main]{neurips_2026}

\usepackage[utf8]{inputenc}
\usepackage[T1]{fontenc}
\usepackage{hyperref}
\usepackage{url}
\usepackage{booktabs}
\usepackage{amsfonts}  
\usepackage{dsfont}
\usepackage{mathtools}
\usepackage{amsmath}
\usepackage{amsthm}
\usepackage{nicefrac}  
\usepackage{tabularx}
\usepackage{booktabs}
\usepackage{graphicx}
\usepackage{algorithm}
\usepackage{algpseudocode}
\usepackage{microtype}
\usepackage{xcolor}
\usepackage{array}
\usepackage{multicol}
\usepackage{multirow}
\usepackage{subcaption}
\usepackage{caption}
\usepackage{soul}
\usepackage{booktabs}
\usepackage{subcaption}
\usepackage{xltabular}
\usepackage{graphicx}
\usepackage{array}
\usepackage{ragged2e}
\newtheorem{theorem}{Theorem}
\newtheorem{lemma}{Lemma}
\newtheorem{corollary}{Corollary}

\newtheorem{assumption}{Assumption}

\newtheorem{definition}{Definition}

\usepackage{mdframed} 
\newcommand{\myNum}[1]{(\emph{#1})}
\newcommand{\smartparagraph}[1]{\vspace{1pt} \noindent {\bf #1}}

\iflatexml
\newcommand{\contribparagraph}[2]{\emph{(#1)} \textbf{#2}}
\else
\newcommand{\contribparagraph}[2]{%
  \vspace{1pt}\noindent \myNum{#1}\ {\bf #2}}
\fi
\newcommand{\inLineComment}[1]{}

\usepackage[most]{tcolorbox}

\usepackage[colorinlistoftodos,bordercolor=orange,backgroundcolor=orange!20,linecolor=orange,textsize=scriptsize]{todonotes}

\definecolor{grey}{gray}{0.5}

\definecolor{codeblue}{rgb}{0.21,0.49,0.74}
\definecolor{codegreen}{rgb}{0,0.6,0}
\newcommand{\greenup}{\textcolor{codegreen}{$\uparrow$}}

\newif\ifshowzfoedits
\showzfoeditstrue 

\expandafter\def\expandafter\normalsize\expandafter{%
    \normalsize%
    \setlength\abovedisplayskip{0pt}%
    \setlength\belowdisplayskip{4pt}%
    \setlength\abovedisplayshortskip{-1.2pt}%
    \setlength\belowdisplayshortskip{1pt}%
}
\usepackage{comment}

\title{

Trust the Direction, Search the Step: Zero-and-First-Order Methods for LLM Fine-Tuning
}

\author{%
  Cristian McGee \\
  University of Central Florida\\
  Orlando, FL, USA
  \And
  El Houcine Bergou \\
  Mohammed VI Polytechnic University\\
  Ben Guerir, Morocco
  \And
  Aritra Dutta \\
  University of Central Florida\\
  Orlando, FL, USA
}

\begin{document}


\maketitle

\begin{abstract}
Step-size selection remains a central challenge in large-scale neural network optimization; conservative steps slow convergence, while aggressive steps can destabilize it. We combine \textbf{Z}ero-and-\textbf{F}irst-\textbf{O}rder optimization~(ZFO) and propose a framework that decouples direction selection from step-size. ZFO uses a trusted first-order optimizer to determine the direction and performs zeroth-order evaluations only along this one-dimensional subspace to choose how far to move. Using the current {gradient information} and two additional objective function evaluations, ZFO instances construct a local model of the objective function along the proposed direction and select a curvature-aware step within a bounded search interval. This yields an adaptive step-selection mechanism that costs less than a full line search.  We provide theoretical guarantees to show that shared-sample evaluations produce reliable finite-difference curvature estimates, that the induced local model selects a near-optimal step along the search interval, and that ZFO converges to a neighborhood of a stationary point. Across the evaluated settings, language models and datasets, ZFO frequently improves optimization and final performance relative to fixed-step first-order baselines, with the magnitude and preferred local model depending on the objective. Our code is publicly available at: \url{https://github.com/nizswan/Zeroth-First-Order-Framework}.
\end{abstract}

\input{Sections/01-introduction}

\input{Sections/03-method}

\input{Sections/04-theory-new}

\input{Sections/05-experiments}

\input{Sections/06-conclusion}

\smartparagraph{Acknowledgment.} During this work, Aritra Dutta was partially supported by the Florida Department of Health Grant AWD00007072. 

\bibliographystyle{plainnat}
\bibliography{references}
 \clearpage

\input{appendix}

\end{document}

%% file: Sections/01-introduction.tex
\section{Introduction}\label{sec:introduction}

Gradient-based first-order (FO) optimization methods are the main workhorse of modern deep neural network training \cite{dutta2023demystifying, adamw, muon_og}. During training, gradient information provides the \emph{direction to move along}, which, coupled with a step size (or learning rate), updates the model parameters. 
While the gradient direction is generally informative, the step size used along that direction is typically chosen through schedules, heuristics, or hyperparameter tuning~\cite{sgdr, goyal, smith2017cyclical, smith2018disciplined}. 
For larger networks, whose local geometry can change rapidly during adaptation, a poorly chosen step size may lead to unstable training~\cite{gradstab}. Despite these step-size instabilities, FO methods remain the standard choice for neural network adaptation because the FO gradient provides strong directional guidance. 

In contrast, zeroth-order (ZO) methods~\cite{mezo, cars, nesterov2017random, stochasticfzo, STP, IS-STP, SMTP, MiSTP} do not compute gradients. They construct approximate gradient updates using only objective function evaluations, often via finite differences along random directions. Although ZO methods reduce computational overhead, gradient approximations are often noisy and high-variance in large-dimensional settings~\cite{balasubramanian, query}. However, when FO provides the update direction, the high-dimensional problem becomes a univariate step-selection problem; see Figure~\ref{fig:zfo}. Consequently, computationally inexpensive ZO evaluations become substantially informative for modeling local function behavior. This raises a natural question: 

\iflatexml
\begin{quote}
\emph{If FO methods are strong at choosing an update direction, can ZO evaluations be used to decide how far to move along that direction?}
\end{quote}
\else
\begin{tcolorbox}[
    colback=grey!5,
    colframe=grey!75!black,
    boxrule=0.8pt,
    arc=2mm,
    left=2pt,
    right=2pt,
    top=1pt,
    bottom=1pt
]
\emph{If FO methods are strong at choosing an update direction, can ZO evaluations be used to decide how far to move along that direction?}
\end{tcolorbox}
\fi

\begin{figure}
    \centering
    \scalebox{0.98}{
    \includegraphics[width=0.99\linewidth]{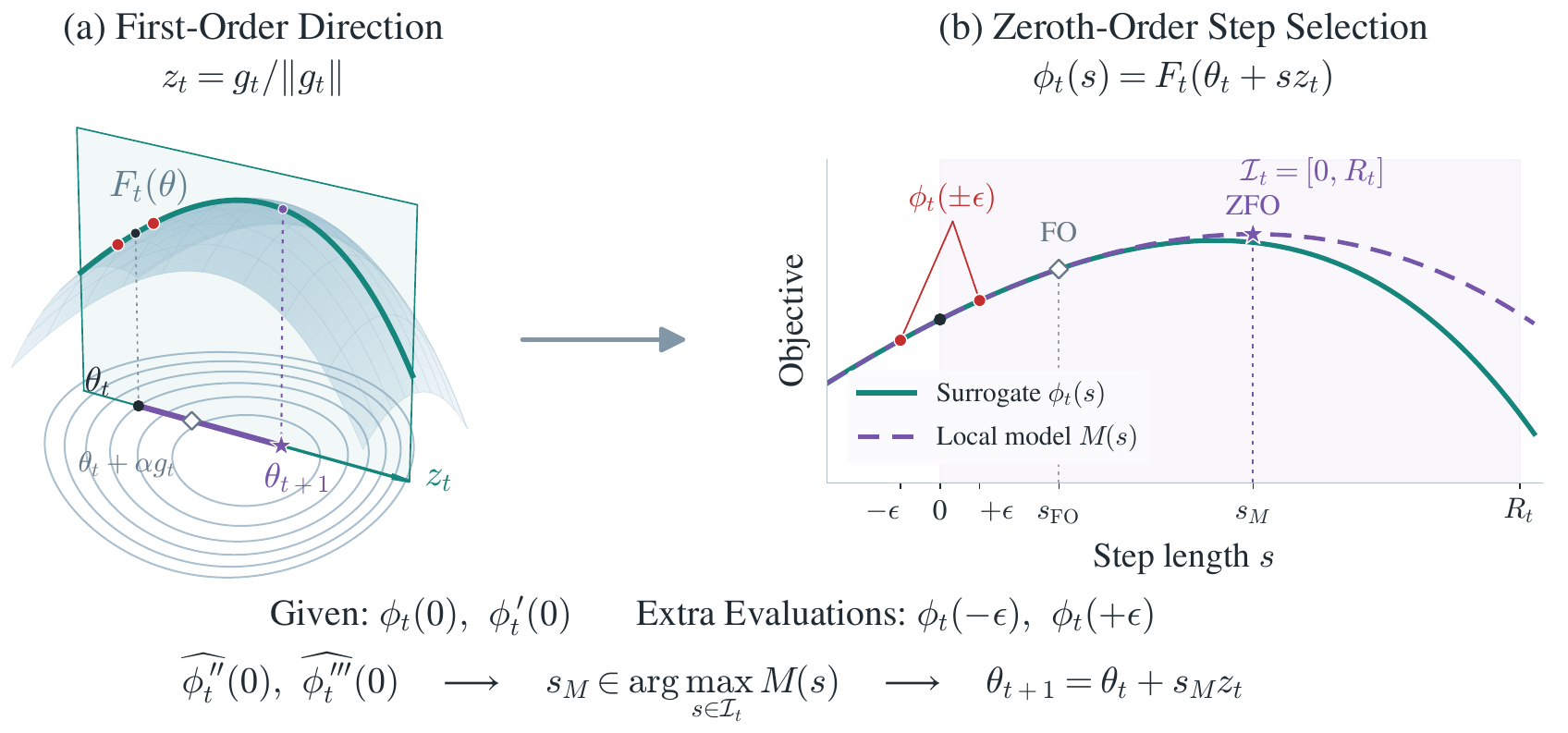}}
    \vspace{-2mm}
    \caption{\small{\textbf{The \colorbox{red!60}{Zero}-and-\colorbox{codegreen!60}{First}-Order (\textcolor{red}{Z}\textcolor{codegreen}{F}O) Framework.} \emph{Left:} A first-order optimizer proposes the direction $z_t$ at $\theta_t$, which restricts the objective to the line $\theta_t + s z_t$. \emph{Right:} Along this line, $\theta_t + s z_t$, the forward pass gives $\phi_t(0)$, the backward pass gives $\phi_t'(0)$, and two symmetric probes at $\pm\epsilon$ supply curvature estimates. ZFO fits a local model $M(s)$ and takes the optimized step $s_M$ over the interval $[0, R_t]$; see Algorithm \ref{alg:zfo}.}}\label{fig:zfo}
    \vspace{-4mm}
\end{figure}

We investigate this question by using ZO evaluations not to replace gradients, but to select reliable step sizes. 
This idea is related to recent work on generalized Newton step selection (GeN)~\cite{bu2025gradient}, which uses function evaluations along a first-order direction to construct a curvature-informed step. In contrast, ZFO takes a broader view: rather than committing to a particular step construction, we formulate fixed-budget directional step selection as a local modeling problem. 
This perspective is not restricted to a specific optimization setup; in principle, the framework applies whenever an approximate gradient direction is available, but its appropriate magnitude is uncertain, a pattern that recurs in existing step-size rules~\cite{mahsereci2017probabilistic}. This uncertainty becomes especially relevant in noisy or stochastic optimization settings.

Large language model (LLM) fine-tuning is a natural setting for this hybridization. Modern fine-tuning pipelines construct minibatches of prompt-response examples and optimize a surrogate objective on these fixed examples. Once a minibatch is formed, nearby parameter evaluations can reuse the same data and require only additional forward evaluations of the perturbed model. This makes ZO probes especially useful---not as a replacement for gradients, but as a cheap mechanism for modeling objective variation along a trusted FO direction. In contrast, fully ZO LLM fine-tuning methods use forward evaluations to estimate high-dimensional gradients~\cite{mezo, cars}. ZFO exploits this structure by assigning direction selection to a base FO optimizer, such as AdamW~\cite{adamw} or Muon~\cite{muon_og}, while using a small number of ZO function evaluations to build a univariate local model of the objective function and select the step size along that direction. Taken together, the ZFO framework acts as a lightweight step-selection wrapper around a standard FO optimizer. This is particularly valuable in LLM fine-tuning, where local geometry can shift rapidly and poorly tuned learning rates can cause unstable or slow training~\cite{lee2026learning, springer2025overtrained}.

\smartparagraph{Problem Setup.}
Let fine-tuning an LLM be represented by a conditional distribution $\pi_\theta(y \mid x)$, where $x$ denotes an input prompt, $y$ denotes a model response, and $\theta \in \mathbb{R}^d$ are the model parameters. We initialize the model from a pretrained reference model and further optimize it using task-specific feedback. We summarize the main notation used throughout the paper in \S\ref {app:notation}; see Table~\ref{tab:notation}.

At a given iteration $t$, a FO fine-tuning algorithm constructs a minibatch of information, denoted by $S_t$. This minibatch contains all quantities needed to define the current training objective. Depending on the method, $S_t$ may include prompts, sampled responses, supervised labels, etc. Conditioned on $S_t$, the fine-tuning method defines a sampled surrogate objective optimized during the current update
$
F_t(\theta)
=
F(\theta; S_t).
$
Different LLM fine-tuning algorithms instantiate this surrogate differently. Our framework does not require a particular choice of $F_t$, and only assumes that the surrogate can be evaluated on the fixed minibatch $S_t$.

This fixed-batch viewpoint is standard in LLM fine-tuning \cite{mezo, schulman2017proximal, cars, lora}. Once $S_t$ is constructed, it is fixed within the update. Evaluating $F_t$ at a nearby parameter value therefore does not require constructing a new batch or collecting new feedback; it only requires re-scoring the same prompt-response pairs under the perturbed model. 

A standard fine-tuning step computes a stochastic gradient of the sampled surrogate at the current parameters, and passes it through an optimizer such as AdamW. If we denote the resulting update direction by $g_t$, the usual first-order gradient ascent update is:
$
\theta_{t+1}
=
\theta_t
+
\alpha g_t,
$
where $\alpha > 0$ is the learning rate. 
This update contains two decisions: the direction $g_t$ and the step length controlled by $\alpha$. The chosen fine-tuning surrogate and the optimizer determine the direction. The step length, however, is usually specified by a learning-rate schedule or hyperparameter tuning~\cite{maclaurin2015gradient}. In LLM fine-tuning, this scalar choice can be sensitive; excessively small steps may fail to fully exploit the information, while aggressive steps can destabilize training~\cite{suk2025optimization}. 

We improve the step-length decision while leaving the optimizer direction unchanged and formulate all objectives as maximization problems. Let $g_t \neq 0$, and define the normalized direction
\begin{equation}\label{eq:gradient_direction}
z_t = \frac{g_t}{\|g_t\|}.
\end{equation}
We restrict the current surrogate to the one-dimensional path along this direction:
$$
\phi_t(s)
=
F_t(\theta_t + s z_t),
\quad {\rm for\;}s \in \mathbb{R}.
$$
The standard optimizer step corresponds to
$
s_{\mathrm{FO}}
=
\alpha \|g_t\|.
$
Thus, choosing the learning rate is equivalent to choosing a scalar displacement along the fixed-batch improvement of $\phi_t$.

Rather than relying only on a global learning-rate schedule, we ask whether the current surrogate $F_t$ supports a better step along the direction already selected by the optimizer. We search over a bounded interval
$
\mathcal{I}_t
=
[0, R_t],
\quad
R_t
=
\beta \alpha \|g_t\|,
$
where $\beta > 0$ controls the search radius relative to the baseline optimizer step. When $\beta = 1$, the interval ends at the usual first-order update; larger values allow the method to consider longer steps while remaining constrained to the same direction:
\begin{equation}
\label{sstar}
s_t^\star
\approx
\arg\max_{s \in \mathcal{I}_t}
\phi_t(s).
\end{equation}
A full line search would evaluate $\phi_t$ at many candidate values of $s$, which is costly for LLMs. Instead, we use a fixed-budget step-selection procedure: the current value $\phi_t(0)$ is obtained from the usual forward evaluation of the surrogate on $S_t$, and the backward pass gives the directional derivative
$
\phi_t'(0)
=
\left\langle
\nabla F_t(\theta_t),
z_t
\right\rangle.
$
We then evaluate the same fixed-batch surrogate at two symmetric perturbations along the update direction:
$$
\phi_t(\epsilon)
=
F_t(\theta_t + \epsilon z_t),
\quad
\phi_t(-\epsilon)
=
F_t(\theta_t - \epsilon z_t), \text{ for a small }\epsilon > 0.
$$
Together, these quantities
$\left\{\phi_t(0),\phi_t'(0), \phi_t(\epsilon), \phi_t(-\epsilon) \right\},
$
define a local model of the surrogate improvement along the optimizer direction. The selected step $s_t$ is obtained by maximizing this model over $\mathcal{I}_t$, and the parameters are updated as
$
\theta_{t+1}
=
\theta_t
+
s_t z_t.
$
In this formulation, the base LLM fine-tuning algorithm determines the surrogate $F_t$ and the update direction $g_t$, while our method determines how far to move along that direction; see Figure \ref{fig:zfo}. The procedure is therefore compatible with a broad class of LLM fine-tuning objectives. Due to space limitations, we defer the related work to \S\ref{sec:related_work}.



{
\smartparagraph{Our contributions.}
We develop ZFO as a general framework that separates first-order direction selection from low-query zeroth-order step selection. This formulation allows different local step-selection models to be developed and analyzed within the same optimization framework. Below, we summarize our main contributions:
}

\contribparagraph{i}{A ZFO optimization framework.}
    We propose the general ZFO framework in which a standard first-order optimizer proposes an update direction, and a small number of zeroth-order evaluations are used to choose an improved step size along this direction. Within this framework, we develop four practical ZFO instances based on Taylor and Padé local models. Whereas GeN represents a specific step-selection rule, ZFO provides a more general fixed-budget directional modeling framework. GeN can be interpreted within the ZFO framework as a particular fixed-budget directional step-selection rule. 

\contribparagraph{ii}{Theoretical justification of adaptive step selection.}
   Our theoretical analysis of the one-dimensional objective follows from restricting the training objective to the first-order search direction. We show that the finite-difference curvature estimates with common random numbers (CRNs) concentrate around their population counterparts with high probability. We further prove that maximizing the resulting local model yields a step whose objective value is near-optimal within a bounded search region. Finally, we establish a global stationarity guarantee showing that ZFO converges to a stationary neighborhood whose size is controlled by the accumulated step-selection error.

\contribparagraph{iii}{Empirical validation across scales.}
    We evaluate ZFO on both controlled low-dimensional problems and LLM fine-tuning tasks. The experiments show that ZFO is robust to the choice of initial learning rate and can match or outperform standard first-order and zeroth-order baselines on most benchmarks. These results demonstrate that lightweight zeroth-order step adaptation can improve the stability and reliability of first-order fine-tuning without incurring the cost of full line search or extensive hyperparameter tuning.

%% file: Sections/03-method.tex
\vspace{-1mm}
\section{Methodology}\label{sec:methodology}
\vspace{-1mm}
We fix an iteration $t$ and suppress the subscript $t$ for readability. We use this convention throughout the section.

\smartparagraph{Local derivative estimates.} We recall that the forward-backward pass provides $\phi(0)$ and the directional derivative,
$
\phi'(0)
=
\langle \nabla F(\theta), z\rangle.
$
ZFO adds two symmetric fixed-batch evaluations,
$$
\phi(\epsilon)
=
F(\theta+\epsilon z),\;{\rm and\;}
\phi(-\epsilon)
=
F(\theta-\epsilon z).
$$
We estimate the second directional derivative by the central finite difference
\[
\widehat{\phi''}(0)
=
\frac{
\phi(\epsilon)-2\phi(0)+\phi(-\epsilon)
}{
\epsilon^2
}.
\]
Using the available first-order derivative, the same two probes also yield
\[
\widehat{\phi'''}(0)
=
\frac{3}{\epsilon^3}
\left(
\phi(\epsilon)-\phi(-\epsilon)-2\epsilon \phi'(0)
\right).
\]
For compactness, define: 
$
a = \phi'(0),b = \widehat{\phi''}(0),$ and $c = \widehat{\phi'''}(0).
$
ZFO uses these quantities to construct a local approximation of $\phi$ (see Figure \ref{fig:zfo}), considering both polynomial models derived from Taylor expansions and rational models derived from Padé approximations.

\contribparagraph{a}{Taylor models.} The second-order Taylor model is
$
T_2(s)
=
\phi(0)
+
a s
+
\frac{1}{2}b s^2.
$
Since $T_2$ is quadratic, the maximizer is obtained by checking the endpoints of $I$ and, when feasible, the stationary point
$
s
=
-{a}/{b}.
$
Similarly, the third-order Taylor model is
$
T_3(s)
=
\phi(0)
+
a s
+
\frac{1}{2}b s^2
+
\frac{1}{6}c s^3. 
$
The candidate maximizers are the endpoints of $I$ and the real roots inside $I$ of
$
a + b s + \frac{1}{2}c s^2 = 0.
$


\contribparagraph{b}{Padé models.}
The Padé variants approximate $\phi$ using rational functions. The second-order Padé model is
\[
P_2(s)
=
\phi(0)
+
\frac{a s}{1+q_2 s},
\quad
q_2
=
-\frac{b}{2a},
\]
Similarly, the third-order Padé model is
\[
P_3(s)
=
\frac{\phi(0)+(a+\phi(0)q_3)s}{1+q_3s+r_3s^2},
\text{ where }
q_3
=
\frac{\phi(0)c-3ab}{3(2a^2-\phi(0)b)},
\quad
r_3
=
\frac{3b^2-2ac}{6(2a^2-\phi(0)b)}.
\]
If a Padé denominator has a pole inside $I$, or if its coefficients are ill-defined, ZFO rejects the Padé step and falls back to the quadratic Taylor step. This \emph{safeguard} avoids unstable rational extrapolation near singularities.

\smartparagraph{Model selection and update.}~Each ZFO variant corresponds to one model choice from
$
\{T_2, T_3, P_2, P_3\}.
$
For a chosen model $M$ in this set, ZFO computes the step by solving
\vspace{-1mm}
\iflatexml
\begin{equation}\label{zfostep}
s_M
\in
\arg\max_{s \in I}
M(s).
\end{equation}
\else
\begin{tcolorbox}[
    colback=grey!5,
    colframe=grey!75!black,
    boxrule=0.8pt,
    arc=2mm,
    left=1pt,
    right=1pt,
    top=0.8pt,
    bottom=0.8pt
]
\begin{equation}\label{zfostep}
s_M
\in
\arg\max_{s \in I}
M(s).
\end{equation}
\end{tcolorbox}
\fi


\begin{algorithm}[t]
\small
\vspace{-1mm}
\caption{\colorbox{red!60}{Zero}-and-\colorbox{codegreen!60}{First}-Order (\textcolor{red}{Z}\textcolor{codegreen}{F}O) Framework}
\label{alg:zfo}
\begin{algorithmic}[1]
\Require  $\theta_0$, base optimizer, learning rate $\alpha$, $\epsilon>0$, $\beta>0$, model $M \in \{T_2,T_3,P_2,P_3\}$
\For{$t=0,1,2,\ldots$}
    \State \colorbox{codegreen!60}{Compute the direction $g_t$ using the base optimizer} \Comment{\texttt{{First-order direction}}}
    \State Set $z_t = g_t/\|g_t\|$ and $I_t = [0,\beta \alpha \|g_t\|]$
    \State Compute: $a_t=\phi_t'(0),
    b_t =
    \frac{\phi_t(\epsilon)-2\phi_t(0)+\phi_t(-\epsilon)}{\epsilon^2},
    c_t =
    \frac{3}{\epsilon^3}
    \left(
    \phi_t(\epsilon)-\phi_t(-\epsilon)-2\epsilon a_t
    \right)$
    \State Construct the chosen local model $M$ using $a_t,b_t,c_t$ 
    \If  {$M$ is a Pad\'e model and has a pole in $I_t$}
    \State Replace $M$ by $T_2$
    \EndIf
    \State Select \colorbox{red!60}{$\displaystyle s_t=s_M \in \arg\max M(s)$, s.t.  $s\in I_t$}\Comment{\texttt{{Zeroth-order step selection}}} 
    \State Update $\theta_{t+1}=\theta_t+s_t z_t$ \Comment{\texttt{Parameter update}} 
\EndFor
\end{algorithmic}
\end{algorithm}
\vspace{-1mm}

All ZFO variants use the same direction, surrogate, and feasible interval. They differ only in how they convert two additional evaluations into a scalar step length. {GeN~\cite{bu2025gradient} also fits within this abstraction as a specific quadratic ZFO instance. GeN uses two function evaluations along the first-order direction to construct a generalized Newton step. Our instances reuse the directional derivative already available from the backward pass so that the same two probes recover both second- and third-order directional information and support Taylor and Padé models. GeN uses curvature-based rejection as its safeguard, while the instances studied here optimize over a bounded interval and fall back when a Padé model is ill-defined.} 
ZFO's novelty is the fixed-budget local-model abstraction, its family of Taylor/Padé instantiations, and the associated unified analysis and empirical validation. Algorithm \ref{alg:zfo} summarizes our framework.

%% file: Sections/04-theory-new.tex
\vspace{-1mm}
\section{Theoretical Analysis}\label{sec:theory}
\vspace{-1mm}
This section provides a theoretical justification for the ZFO step-selection framework. The analysis has three components. First, we show that shared-batch evaluations yield reliable finite-difference estimates of directional curvature; see \S\ref{sec:loacl_convergence} \myNum{i}. Second, conditional on the sampled examples, we show that maximizing the local ZFO model returns a step whose objective value is close to the best value attainable over the search interval; see \S\ref{sec:loacl_convergence} \myNum{ii}. Third, we connect this local step-selection property to a global stationarity guarantee for the ZFO update sequence; see \S\ref{sec:global_convergence}. We omit the iteration index $t$ when no ambiguity arises.

\vspace{0mm}
\subsection{Local Convergence Analysis}\label{sec:loacl_convergence}
\vspace{0mm}

\smartparagraph{Shared-batch common-random-numbers (CRN).} 
ZFO evaluates all local objective values on the same minibatch. At a fixed iteration, the examples are sampled once and then reused across the evaluations
$\left\{\phi(-\epsilon),
\phi(0),
\phi(\epsilon)\right\}.
$
After conditioning on these shared examples, $\phi$ becomes a deterministic function of the scalar step length, $s$. This conditional determinism is the basis of the local approximation analysis. 
Let
$
\tau
=
\{\tau_1,\ldots,\tau_n\}
$
denote the shared sampled examples used under CRN. For each example $\tau_i$, define the per-sample directional objective 
$
\phi_i(s)
=
F(\theta+s z;\tau_i).
$
The empirical shared-batch CRN objective is 
$
\phi(s)
=
\frac{1}{n}
\sum_{i=1}^{n}
\phi_i(s).
$
\begin{definition}[Population directional objective]
The population counterpart of the empirical CRN objective is
$
\psi(s)
=
\mathbb E_{\tau}
[
\phi_i(s)
],
$
where the expectation is taken with respect to the distribution used to generate the empirical objective.
\end{definition}
\contribparagraph{i}{Local smoothness assumptions.}
The first group of assumptions ensures that the directional objective admits Taylor expansions with controlled remainders.

\begin{assumption}[Pathwise smoothness]
\label{ass:pathwise-smoothness}
For each example $\tau_i$, the per-sample directional objective $\phi_i$ is five times continuously differentiable on the interval $[-\epsilon,R]$:
$
\phi_i
\in
C^5([-\epsilon,R]).
$
\end{assumption}

\begin{assumption}[Higher-order derivative bounds]
\label{ass:derivative-bounds}
There exist constants $W,M,B,L>0$ such that, for every $i$ and every $s\in[-\epsilon, R]$,
$|\phi_i''(s)| \leq W,$
$
|\phi_i'''(s)|
\leq
M,
$
$
|\phi_i^{(4)}(s)|
\leq
B,
$
and
$
|\phi_i^{(5)}(s)|
\leq
L.
$
\end{assumption}


{
\contribparagraph{ii}{Sampling assumptions.}
Population concentration is analyzed at a fixed iterate and direction under the independent-example assumptions below. Algorithm~\ref{alg:zfo} reuses its batch across probes, but also computes the direction from that batch; CRN alone does not establish the independence needed by this concentration argument. Thus Theorem~\ref{thm:crn-curvature-concentration} is a fixed-direction result, whereas the approximation guarantee in Theorem~\ref{thm:second-order-local-bound} is pathwise once the batch and direction are fixed. Bounded directional derivatives imply the sub-Gaussian condition by Hoeffding's lemma; see~\eqref{eq:hoeffding_subgaussian}.
}


\begin{assumption}[CRN]
\label{ass:crn} For each index $i$, the same example $\tau_i$ is used to evaluate $
\{\phi_i(-\epsilon), \phi_i(0), \phi_i(\epsilon)\}.
$
Moreover, for $i\neq j$, the examples $\tau_i$ and $\tau_j$ are sampled independently.
\end{assumption}

\begin{assumption}
\label{ass:subgaussian}
    (Sub-Gaussian Deviation) For each $\phi_{i},$ the directional deviations at $s=0$ are sub-Gaussian up to third order, i.e., there exists a constant $\sigma$ such that, for every $k\in\{0,1,2,3\}$ and all $\lambda\in\mathbb{R},$ 
    $\mathbb{E}[\exp(\lambda[\phi_{i}^{(k)}(0)-\psi^{(k)}(0)])] \le \exp(\frac{\lambda^2\sigma^2}{2}).$
\end{assumption}


\begin{theorem}[CRN curvature concentration]
\label{thm:crn-curvature-concentration}
Suppose Assumptions~\ref{ass:pathwise-smoothness}--\ref{ass:subgaussian} hold. Then, with probability at least $1-\delta$,
$
\left|
\widehat{\phi''}(0)-\psi''(0)
\right|
\leq
\frac{B\epsilon^2}{12}
+
\sigma
\sqrt{
\frac{2\log(4/\delta)}{n}
},
$
and
$
\left|
\widehat{\phi'''}(0)-\psi'''(0)
\right|
\leq
\frac{L\epsilon^2}{20}
+
\sigma
\sqrt{
\frac{2\log(4/\delta)}{n}
}.
$
\end{theorem}

Theorem~\ref{thm:crn-curvature-concentration} shows that the finite-difference curvature estimates used by ZFO concentrate around their population counterparts. Each bound contains two terms. The first term is a deterministic finite-difference bias controlled by the perturbation scale $\epsilon$ and the relevant higher-order derivative bound. The second term is a stochastic sampling term that decreases at the classic rate $1/\sqrt{n}$. Thus, under CRN, the two additional objective evaluations provide statistically meaningful curvature information.

After conditioning on the shared examples, the empirical objective $\phi$ is deterministic. 
Let
$
s_M
\in
\arg\max_{s\in[0, R]}
T_2(s)
$
denote the model-selected step by $T_2$, and let
$
s_{\phi}
\in
\arg\max_{s\in[0,R]}
\phi(s)
$
denote the best step for the empirical CRN objective over the same interval.

\begin{theorem}[{Second-order Taylor approximation bound}]
\label{thm:second-order-local-bound}
Suppose Assumptions~\ref{ass:pathwise-smoothness} and~\ref{ass:derivative-bounds} hold. Then the step selected by maximizing the second-order Taylor model satisfies
$
\phi(s_M)
\geq
\phi(s_{\phi})
-
\frac{M R^3}{3}
-
\frac{B\epsilon^2 R^2}{12}.
$
\end{theorem}

Theorem~\ref{thm:second-order-local-bound} shows that, once the examples are fixed, ZFO selects a step that is close to the best step available in the search interval. The first error term,
$
\frac{M R^3}{3},
$
is the deterministic Taylor approximation error over the interval. The second error term,
$
\frac{B\epsilon^2 R^2}{12},
$
is the error caused by estimating curvature through finite differences. The guarantee is strongest when the search radius $ R$ is moderate, the perturbation scale $\epsilon$ is small, and the directional objective has limited higher-order variation. The same proof strategy extends to the third-order Taylor model. For the Padé variants, the local model is rational rather than polynomial. These models can better capture nonlinear curvature changes over a finite interval. Their approximation guarantees follow from the same maximizer-comparison argument, replacing the Taylor remainder by a rational residual term. We defer the analysis of these variants to \S\ref{subsec:zfo_analysis}. 

\vspace{-1mm}
\subsection{Global Convergence Analysis}\label{sec:global_convergence}
\vspace{0mm}

We now connect the local analysis to the true optimization objective $J$. Recall that ZFO optimizes on the surrogate $F_t,$ whose expectation under the current sampling distribution is $\bar F_t(\theta + sz_t)=\psi(s)$. For instance, standard policy optimization~\cite{REINFORCE} and its modern extensions~\cite{schulman2015trust, schulman2017proximal} construct surrogates such that $\nabla J(\theta_t)=\nabla \bar F(\theta_t)$ at the current iterate. Define the expected directional alignment quantity
$
q_t = \mathbb{E}_t[\langle \nabla J(\theta_t), z_t \rangle].
$ We start by stating two main classical assumptions that we use in the sequel. 


\begin{assumption}[Smooth objective]
\label{ass:smooth-bounded}
The true objective $J:\mathbb R^d\to\mathbb R$ has an $\mathcal L_J$-Lipschitz gradient and is bounded above by $J^\star$. Consequently, for all $x,y$,
$
J(y)
\geq
J(x)
+
\langle
\nabla J(x),
y-x
\rangle
-
\frac{\mathcal L_J}{2}
\|y-x\|^2,
$
and
$
J(\theta_t)
\leq
J^\star
$
for all iterations $t$. Additionally, the expected surrogate $\bar F$ is also $\mathcal L_{\bar F}$-smooth.
\end{assumption}

\begin{assumption}[Aligned update direction]
\label{ass:alignment}
There exists $\mu>0$ such that, for all $t$,
$
q_t=\mathbb{E}_t[\langle
\nabla J(\theta_t),
z_t
\rangle]
\geq
\mu
\|\nabla J(\theta_t)\|.
$
\end{assumption}

Assumption~\ref{ass:alignment} is natural since gradient methods estimate an ascent direction for the  objective~\cite{schulman2017proximal}. Although optimizers such as AdamW rescale this direction, the resulting update remains positively oriented with $\nabla J(\theta_t)$ in expectation~\cite{pmlr-v80-balles18a}, supporting the assumption of $\mu > 0$. 

We begin with a lemma that bounds the step-selection error at iteration $t$ of the objective:
$
\max_{s\in I_t}
J(\theta_t+s z_t)
-
J(\theta_t+s_t z_t).
$ This quantity measures the loss in objective value by using the ZFO step $s_t$ instead of the best step available in the interval $I_t$ for the true objective $J$. We assume that the value and slope concentration bounds hold jointly for the batch-selected direction $z_t$, conditional on the history before iteration $t,$ with probability at least $1-\delta,$ with $0<\delta<3\mu^2/4.$

\begin{lemma}[Surrogate vs. Empirical Difference]\label{lemma:surrogate-vs-empiric}
    Let Assumptions~\ref{ass:pathwise-smoothness}-\ref{ass:smooth-bounded} hold, then 
    $\max_{s\in I_t}
J(\theta_t+s z_t)
-
J(\theta_t+s_t z_t) \leq \Delta_t,$ holds with probability at least $1- \delta,$ where  $$\Delta_t \coloneq (\mathcal{L}_J+\mathcal{L}_{\bar F})R_t^2+2\sigma(1+R_t)\sqrt{\frac{2\log(4/\delta)}{n}}+2WR_t^2+\phi(s_{\phi}) - \phi(s_M).$$ 
\end{lemma}
This lemma quantifies the loss in objective value when the ZFO step $s_t$, selected through the local model, is used instead of the best step in the interval $I_t$. The error term $\Delta_t$ combines smoothness mismatch, finite-sample concentration error, approximation error, and the model-selection gap $\phi(s_{\phi})-\phi(s_M)$. E.g., when the local model is the second-order Taylor model $T_2$, using Theorem~\ref{thm:second-order-local-bound} we get 
$$\Delta_t \leq (\mathcal{L}_J+\mathcal{L}_{\bar F})R_t^2+2\sigma(1+R_t)\sqrt{\frac{2\log(4/\delta)}{n}}+2WR_t^2+\frac{M R^3}{3}
+ \frac{B\epsilon^2 R^2}{12}.$$

Finally, Theorem~\ref{thm:zfo-stationarity} converts the local quality of the ZFO step-selection rule into a global stationarity guarantee.

\begin{theorem}[Stationarity of the ZFO update]
\label{thm:zfo-stationarity}
Let Assumptions~\ref{ass:pathwise-smoothness}--\ref{ass:alignment} hold. Let $\mathcal K\coloneq\mathcal L_J+\mathcal L_{\bar F}+2W$ and $\chi\coloneq 3\mu^2-4\delta>0$, with $q_t/(2\mathcal K)\in I_t$ almost surely. Then, for every $T>0$, we have
$$
\min_{0\leq t<T}
\mathbb{E}\|\nabla J(\theta_t)\|^2
\leq
\frac{
8\mathcal K(J^\star-J(\theta_0))
}
{\chi T}+
\frac{8\mathcal K \tilde \Delta_T}{\chi}, \text{ where } \tilde \Delta_T=\frac{1}{T}\sum_{t=0}^{T-1} \mathbb{E}[\Delta_t].
$$
\end{theorem}

The upper bound in Theorem \ref{thm:zfo-stationarity} has two terms. The first term,
$
\frac{
8\mathcal K(J^\star-J(\theta_0))
}
{\chi T},
$
is the usual optimization term and decays at rate $1/T$. The second term,
$
\frac{8\mathcal K \tilde \Delta_T}{\chi},
$
is the price paid for imperfect step selection. Thus, ZFO does not necessarily converge exactly to a stationary point unless the average step-selection error vanishes. Instead, it converges to a neighborhood of stationarity controlled by the error
$
\tilde \Delta_T. 
$ 

\newcommand{\numcell}[1]{{\fontsize{6.6}{7.2}\selectfont #1}}
\newcommand{\modelname}[1]{{\fontsize{6.8}{7.4}\selectfont \textbf{#1}}}
\newcommand{\datasetname}[1]{{\fontsize{6.8}{7.4}\selectfont #1}}
\begin{table*}
\centering
\small
\setlength{\tabcolsep}{4pt}
\renewcommand{\arraystretch}{0.9}
\caption{\small{\textbf{Evaluation performance for baseline} first-order (FO) fine-tuning (using AdamW) and zeroth-order (ZO) fine-tuning (using MeZO) against ZFO methods across models and reasoning benchmarks. All runs report the final evaluation and mean results over 3 seeds. \colorbox{green!25}{Green} is the best and \colorbox{red!25}{red} is the second best result.}}
\vspace{-1mm}
\scalebox{0.95}{\begin{tabular}{@{}llcc|cccc@{}}
\toprule
\multirow{2}{*}{\textbf{Model}} &
\multirow{2}{*}{\textbf{Dataset}} &
\multicolumn{2}{c}{\textbf{Baseline}} &
\multicolumn{4}{c}{\textbf{ZFO Methods (This Paper)}} \\

\cmidrule(lr){3-4}
\cmidrule(lr){5-8}
 &  & AdamW (FO) & MeZO (ZO) & Taylor$_2$ & Taylor$_3$ & Padé$_2$ & Padé$_3$ \\
\midrule
\multirow{5}{*}{\modelname{Qwen-2.5-Math-1.5B}}
 & \datasetname{{GSM8K}}      & \numcell{78.24 $\pm$ 0.59} & \numcell{0.00} & \numcell{\colorbox{green!25}{{81.00} $\pm$ {0.11}}} & \numcell{75.56 $\pm$ 0.24} & \numcell{80.95 $\pm$ 0.49} & \numcell{\colorbox{red!25}{80.97 $\pm$ 0.73}} \\
 & \datasetname{MATH}       &  {\numcell{31.77 $\pm$ 1.63}} &  \numcell{22.01 $\pm$ 3.93} & \numcell{\colorbox{green!25}{40.10 $\pm$ 1.19}}   & \numcell{20.18 $\pm$ 3.03} & \numcell{37.50 $\pm$ 1.79} & \numcell{\colorbox{red!25}{39.58 $\pm$ 0.45}}\\
 & \datasetname{SVAMP}      & \numcell{88.11 $\pm$ 0.69} &  \numcell{68.67 $\pm$ 10.33}& \numcell{\colorbox{red!25}{90.78 $\pm$ 0.96}} & \numcell{89.11 $\pm$ 0.51} & \numcell{\colorbox{green!25}{90.89 $\pm$ 1.02}} & \numcell{90.33 $\pm$ 2.19} \\
 & \datasetname{AsDiv}      & \numcell{91.45 $\pm$ 0.36} & \numcell{64.72 $\pm$ 3.51} & \numcell{\colorbox{green!25}{94.09 $\pm$ 0.29}} & \numcell{88.27 $\pm$ 0.35} & \numcell{93.04 $\pm$ 0.22} & \numcell{\colorbox{red!25}{93.46 $\pm$ 0.18}} \\
 & \datasetname{OpenBookQA} & \numcell{26.07 $\pm$ 2.53} & \numcell{8.00 $\pm$ 10.05} & \numcell{{37.80 $\pm$ 17.85}} & \numcell{{39.73 $\pm$ 20.99}} & \numcell{\colorbox{red!25}{60.67 $\pm$ 0.95}} &  \numcell{\colorbox{green!25}{64.07 $\pm$ 1.22}} \\

\midrule

\multirow{3}{*}{\modelname{Phi-2}}
 & \datasetname{SVAMP}      & \numcell{55.44 $\pm$ 4.91} & \numcell{22.44 $\pm$ 3.89} & {\numcell{55.67 $\pm$ 7.86}} & {\numcell{39.33 $\pm$ 0.88}} & {\numcell{\colorbox{red!25}{64.22 $\pm$ 0.19}}} & {\numcell{\colorbox{green!25}{65.11 $\pm$ 5.39}}} \\
 & \datasetname{AsDiv}      & \numcell{56.51 $\pm$ 18.76} & \numcell{10.15 $\pm$ 4.88} & {\numcell{65.42 $\pm$ 2.59}} & \numcell{55.85 $\pm$ 2.55} & {\numcell{\colorbox{red!25}{65.58 $\pm$ 3.74}}} & \numcell{\colorbox{green!25}{69.27 $\pm$ 1.92}} \\
 & \datasetname{OpenBookQA} & \numcell{51.20 $\pm$ 1.11} & \numcell{20.33 $\pm$ 10.90}   & \numcell{\colorbox{green!25}{58.27 $\pm$ 13.62}} & {\numcell{51.27 $\pm$ 0.95}} & {\numcell{49.40 $\pm$ 0.20}} & {\numcell{\colorbox{red!25}{52.07 $\pm$ 0.23}}}\\

\midrule

\multirow{1}{*}{\modelname{Gemma-2-2B}}
 & \datasetname{SVAMP}      & \numcell{52.00 $\pm$ 1.76} & \numcell{14.44 $\pm$ 5.93} & \numcell{\colorbox{green!25}{59.56 $\pm$ 0.51}} & \numcell{47.22 $\pm$ 2.01} &\numcell{\colorbox{red!25}{59.00 $\pm$ 1.76}}  & \numcell{58.22 $\pm$ 3.69} \\

\midrule

\multirow{3}{*}{\modelname{Llama-3.2-1B}}
 & \datasetname{GSM8K}      & \numcell{4.85 $\pm$ 5.47} & \numcell{1.30 $\pm$ 0.98} & \numcell{4.02 $\pm$ 3.98} & \numcell{\colorbox{green!25}{9.00 $\pm$ 6.68}} & \numcell{\colorbox{red!25}{5.05 $\pm$ 6.94}} & \numcell{1.72 $\pm$ 0.88} \\
 & \datasetname{AsDiv}      &  \numcell{16.60 $\pm$ 28.12} & \numcell{3.33 $\pm$ 0.48} & {\numcell{13.72 $\pm$ 23.17}} & \numcell{\colorbox{red!25}{44.77 $\pm$ 1.42}} & {\numcell{{25.23 $\pm$ 21.97}}} & \numcell{\colorbox{green!25}{52.16 $\pm$ 4.75}} \\
 & \datasetname{OpenBookQA} &  {\numcell{25.53 $\pm$ 0.31}} & \numcell{0.00} & \numcell{\colorbox{red!25}{34.40 $\pm$ 1.56}} & \numcell{31.53 $\pm$ 0.81} & \numcell{\colorbox{green!25}{34.60 $\pm$ 1.04}} & \numcell{{25.64 $\pm$ 0.23}} \\
\bottomrule
\label{tab:empirics}
\end{tabular}}
\vspace{-8mm}
\end{table*}



The three results establish the theoretical foundation of ZFO. The CRN concentration theorem establishes the statistical reliability of the finite-difference curvature estimates; the deterministic approximation theorem shows that maximizing the ZFO model yields a near-optimal local step; and the global stationarity theorem shows that controlled step-selection error preserves convergence to stationarity. Together, these results justify the central mechanism of ZFO: a small number of fixed-batch objective evaluations can recover useful local curvature information and improve step selection over a fixed first-order update while retaining a principled convergence interpretation.

%% file: Sections/05-experiments.tex
\vspace{-1mm}
\section{Experiments}\label{sec:experiments}
\vspace{-1mm}

{
\smartparagraph{Organization of the Results.} We organize our experimental results as follows: In the \emph{main paper}, \myNum{i} we compare ZFO against first-order fine-tuning with AdamW and zeroth-order fine-tuning with MeZO across several LLMs and datasets in Table~\ref{tab:empirics}; \myNum{ii} we conduct a focused Qwen-2.5-Math-1.5B sub-study over additional reasoning benchmarks in Table~\ref{tab:qwen_substudy}; \myNum{iii} we evaluate ZFO on classical optimization problems to assess its benefits beyond LLM fine-tuning in Figure~\ref{fig:classic_opt}; \myNum{iv} we summarize the ablation studies in the appendix; and \myNum{v} we report wall-clock and memory overheads in Figures~\ref{fig:time_graph} and~\ref{fig:peak_memory} to quantify the computational cost of ZFO. We defer the remaining experiments and expanded diagnostics to \emph{Appendix}~\ref{sec:add_insights}, where \myNum{vi} we provide the complete six-problem classical optimization study in Table~\ref{tab:optimization_beyond_llms} and Figure~\ref{fig:optimization_beyond_llms}, \myNum{vii} we validate the local ZFO approximants on analytical functions in Figure~\ref{fig:local_functions} and the step-selection mechanism on controlled toy RLVR environments in Figure~\ref{fig:toy_collection}, \myNum{viii} we provide learning-rate sensitivity curves in Figure~\ref{fig:empirical_curves} and study the interaction between $\alpha$ and $\beta$ in Figure~\ref{fig:ablation}, including the complementary $\beta=1$ control in Table~\ref{tab:beta1_llm_ablation}, \myNum{ix} we analyze the frequency of the steps selected by each ZFO method in Table~\ref{tab:step_selection_frequencies}, \myNum{x} we summarize cross-setting improvements in Table~\ref{tab:zfo_win_ratios} and compare against additional \emph{first-order baselines} (Prodigy~\cite{prodigy} and AdamW~\cite{adamw}) in Table~\ref{tab:additional_fo_zfo_extremes}, and \myNum{xi} we evaluate the sequential \texttt{ZFO-Seq} against line-search baselines (Armijo~\cite{nocedal2006numerical}, BB~\cite{barzi}, Polyak~\cite{polyak}, and PLS~\cite{mahsereci2017probabilistic}) in Table~\ref{tab:line_search_selected_results}. Unless otherwise stated, all reported results are averaged over three seeds. The LLM experiments use the implementation described in \S \ref{subsec:objective_formulation}. Additional experimental configuration details are provided in \S\ref{subsec:emp_config}. 
}

\smartparagraph{RLVR, LLMs, and Datasets used for Fine-tuning.} We report the performance of various high-dimensional models configured for an RLVR setup using the GRPO policy gradient~\cite{grpo} and the AdamW optimizer~\cite{adamw}, including Qwen-2.5-MATH-1.5B~\cite{qwen}, Phi-2~\cite{phi2}, Gemma-2-2B~\cite{gemma}, and Llama-3.2-1B~\cite{llama3herd} on the datasets GSM8K~\cite{gsm8k}, SVAMP~\cite{svamp}, AsDiv~\cite{asdiv}, OpenBookQA~\cite{openbookqa}, and a (512/256 train/evaluation) subset of MATH~\cite{math_dataset}; Table~\ref{tab:model-dataset-summary} summarizes the model and dataset details. 

\smartparagraph{Learning-rate selection.} For a fair comparison, we fix the learning rate $\alpha$ to the value that performs the best for the FO baseline, chosen from the set $\{10^{-6}, 5\times 10^{-6}, 10^{-5}\}$, and consequently select the best bound $\beta$ from $\{3, 5, 10 \}$; the remainder of the configuration details are in \S\ref{subsec:emp_config}.


\vspace{-1mm}
\subsection{ZFO for LLM fine-tuning and classic optimization}
\vspace{0mm}



\contribparagraph{i}{Does directional ZO step selection improve FO optimization?}
We show the performance of the ZFO framework on LLM fine-tuning in Table~\ref{tab:empirics}, where at least one ZFO variant improves on the first-order fine-tuning (FO FT) baseline for every model-dataset pair, indicating that ZFO's benefits are not confined to a single model family or benchmark.~We observe the largest gains where the FO baseline is relatively weak or highly variable, including SVAMP for Phi-2 and Gemma-2-2B, as well as AsDiv for Llama and MATH and OpenBookQA for Qwen. When the FO baseline is already strong, such as Qwen on SVAMP, ZFO yields smaller but still positive improvements in mean performance. Although the best-performing ZFO variant changes across tasks and models, Taylor$_2$ provides a strong representative instance of the framework, improving over FO in 14 of the 16 LLM settings summarized in Table~\ref{tab:zfo_win_ratios}. Taylor$_2$ is a strong representative default in our experiments, although the preferred local model remains task-dependent.


\contribparagraph{ii}{Qwen Sub-Study.} Table~\ref{tab:qwen_substudy} extends the main results by isolating Qwen-2.5-Math-1.5B on four reasoning benchmarks: ARC-Challenge~\cite{arc}, StrategyQA~\cite{strategyqa}, FOLIO~\cite{folio}, and CODAH~\cite{codah}. At least one ZFO method outperforms the FT baseline on every dataset, with the largest improvements on ARC-Challenge (\greenup $8.81$ over FO FT) and CODAH (\greenup $4.92$ over FO FT). Even under the best learning rate configurations, the main results and the Qwen sub-study indicate that FO FT exhibits seed variance under our settings. 


\begin{table*}[t]
\centering
\small
\setlength{\tabcolsep}{4pt}
\renewcommand{\arraystretch}{0.9}
\caption{\small{\textbf{Qwen-2.5-Math-1.5B study} on ZFO methods in our settings via additional reasoning benchmarks. Results report evaluation performance averaged over 3 seeds for FO and ZFO methods.}}
\vspace{-4mm}
\begin{tabular}{@{}lccccc@{}}
\multicolumn{5}{c}{} \\
\toprule
\textbf{Method} 
& \textbf{ARC-Challenge} 
& \textbf{StrategyQA} 
& \textbf{FOLIO} 
& \textbf{CODAH} 
\\

\midrule

{AdamW (FO)}      
& {49.29 $\pm$ 16.08} 
& {68.56 $\pm$ 0.91} 
& {34.48 $\pm$ 1.13} 
& \colorbox{red!25}{57.91 $\pm$ 3.24} 
\\
\midrule
{Taylor$_2$} 
& {\colorbox{green!25}{58.10 $\pm$ 1.33}} 
& \colorbox{red!25}{69.63 $\pm$ 1.82} 
& {\colorbox{green!25}{36.60 $\pm$ 3.26}} 
& {\colorbox{green!25}{62.83 $\pm$ 3.14}} 
\\

{Taylor$_3$} 
& {54.24 $\pm$ 0.94} 
& {61.67 $\pm$ 1.46} 
& {34.48 $\pm$ 1.13} 
& {55.28 $\pm$ 1.66} 
 \\

{Padé$_2$}  
& {53.58 $\pm$ 2.77} 
& {65.55 $\pm$ 4.43} 
& {36.27 $\pm$ 2.73} 
& {53.96 $\pm$ 17.72} 
\\

{Padé$_3$}  
& \colorbox{red!25}{57.05 $\pm$ 1.24} 
& {\colorbox{green!25}{70.06 $\pm$ 0.61}} 
& \colorbox{red!25}{36.44 $\pm$ 4.53} 
& {41.85 $\pm$ 20.58} 
\\

\bottomrule
\label{tab:qwen_substudy}
\end{tabular}
\vspace{-6mm}
\end{table*}

{
\contribparagraph{iii}{Classic Optimization.} To show that ZFO is not intrinsically tied to RLVR and to examine ZFO's performance beyond LLM fine-tuning, Figure~\ref{fig:classic_opt} presents three classic optimization problems; see Table~\ref{tab:optimization_beyond_llms} for the comprehensive summary of the six classic optimization problems. At least one ZFO variant achieves a substantially lower median objective than its corresponding FO baseline on every displayed problem after 100 updates. Taylor$_3$--AdamW yields the strongest results on nonlinear least squares and Beale, while Pad\'e$_3$--AdamW performs best on Rosenbrock. Consistent with our LLM experiments, the best approximant varies across problems, indicating that the gains from step selection depend on the local objective.
}

\begin{figure}
    \centering
    \includegraphics[width=0.99\linewidth]{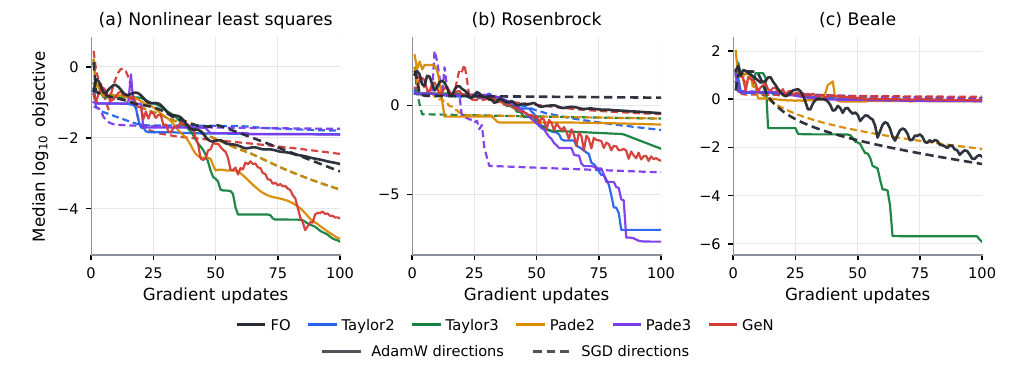}
    \vspace{-2mm}
    \caption{\small{\textbf{Results on classic optimization problems.} We plot the median $\log_{10}$ objective over three perturbed initializations. GeN has a similar computational budget as ZFO, as both methods perform two additional function evaluations per iteration. We study the extended classic optimization problems in Table~\ref{tab:optimization_beyond_llms} and Figure~\ref{fig:optimization_beyond_llms}.}}\label{fig:classic_opt}
    \vspace{-4mm}
\end{figure}

{
\contribparagraph{iv}{Ablation Studies.}
We provide additional ablation studies in Appendix~\ref{sec:add_insights} to better understand when and how ZFO improves step selection. Figure~\ref{fig:empirical_curves} compares ZFO against AdamW across fixed learning rates, showing that gains in Table~\ref{tab:empirics} cannot generally be recovered by simply changing the FO learning rate. Figure~\ref{fig:ablation} further studies the interaction between the learning rate $\alpha$ and search bound $\beta$. Increasing $\beta$ is generally beneficial at smaller learning rates but can become unstable at larger learning rates, consistent with the dependence of the search radius on the product $\alpha\beta$. Table~\ref{tab:step_selection_frequencies} shows that the four ZFO instances use the search interval differently, ranging from primarily interior steps for Taylor$_3$ to mostly endpoint decisions for Padé$_2$.
}

\vspace{-1mm}
\subsection{Computational Cost}\label{subsec:performance}
\vspace{-1mm}

The relative computational overhead of ZFO is small. A standard iteration involves four key components: a rollout call for the trajectory $\tau_t,$ a reward evaluation $R(\tau_t)$, a forward pass, and a backward pass. In contrast, our ZFO methods leave all four operations untouched and only call two additional forward-pass probes. Since these probes are only forward passes, the overhead is small and comparable to the forward-only cost structure of ZO methods such as MeZO~\cite{mezo}. From this perspective, we present both wall-clock time costs in Figure~\ref{fig:time_graph} and memory overhead costs in Figure~\ref{fig:peak_memory}. We train the models on an NVIDIA H100 80GB GPU.

\begin{figure}[t]
    \centering
    \begin{minipage}[t]{0.58\linewidth}
        \centering
        \includegraphics[width=\linewidth]{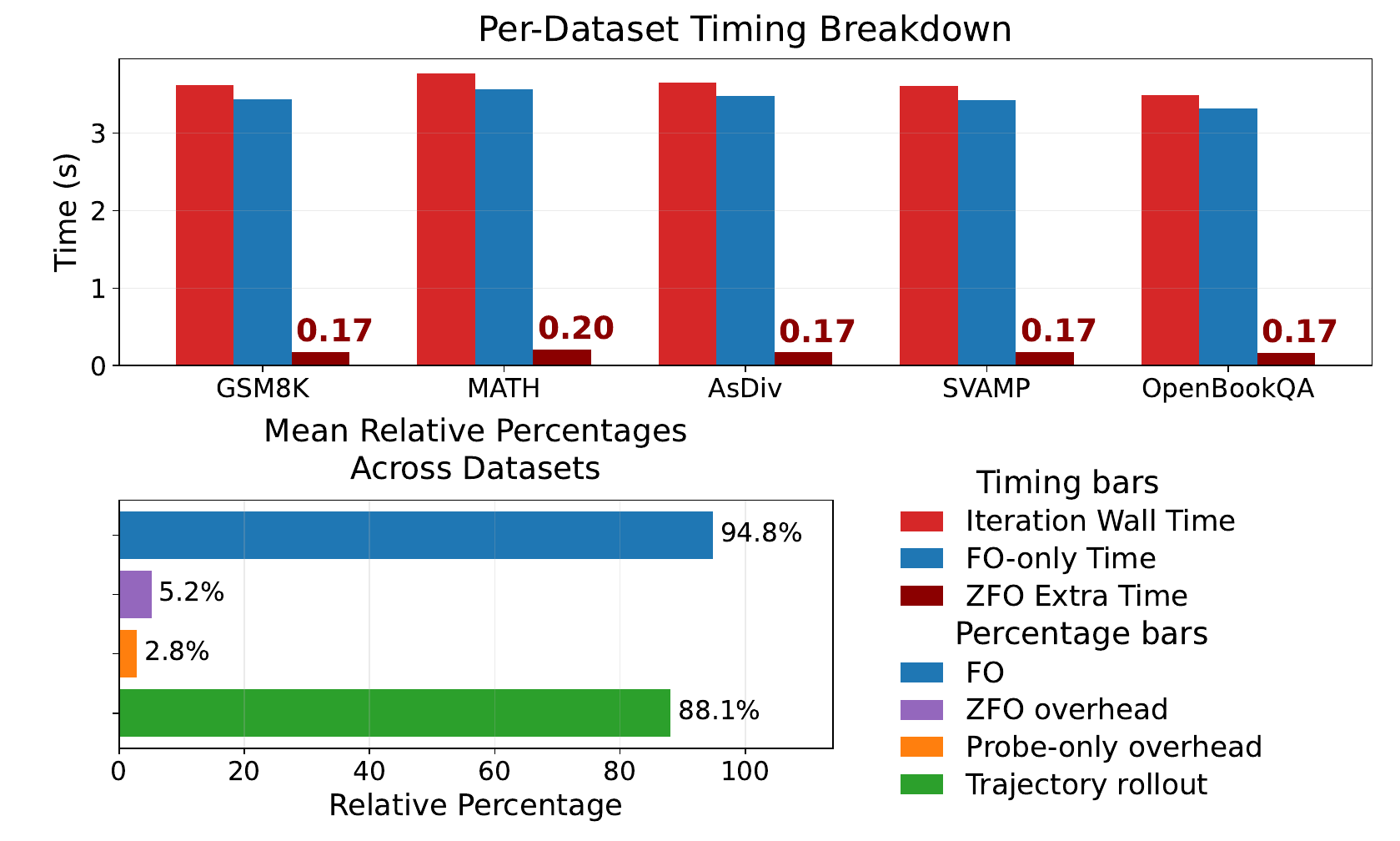}
        \captionof{figure}{\small{\textbf{Wall-clock time comparison.} ZFO step selection adds little wall-clock overhead. Across five benchmarks, trajectory rollout dominates iteration time, while the additional ZFO computation contributes only a small fraction of the time that FO-only updates cost.}}
        \label{fig:time_graph}
    \end{minipage}\hfill
    \begin{minipage}[t]{0.4\linewidth}
        \centering
        \includegraphics[width=\linewidth]{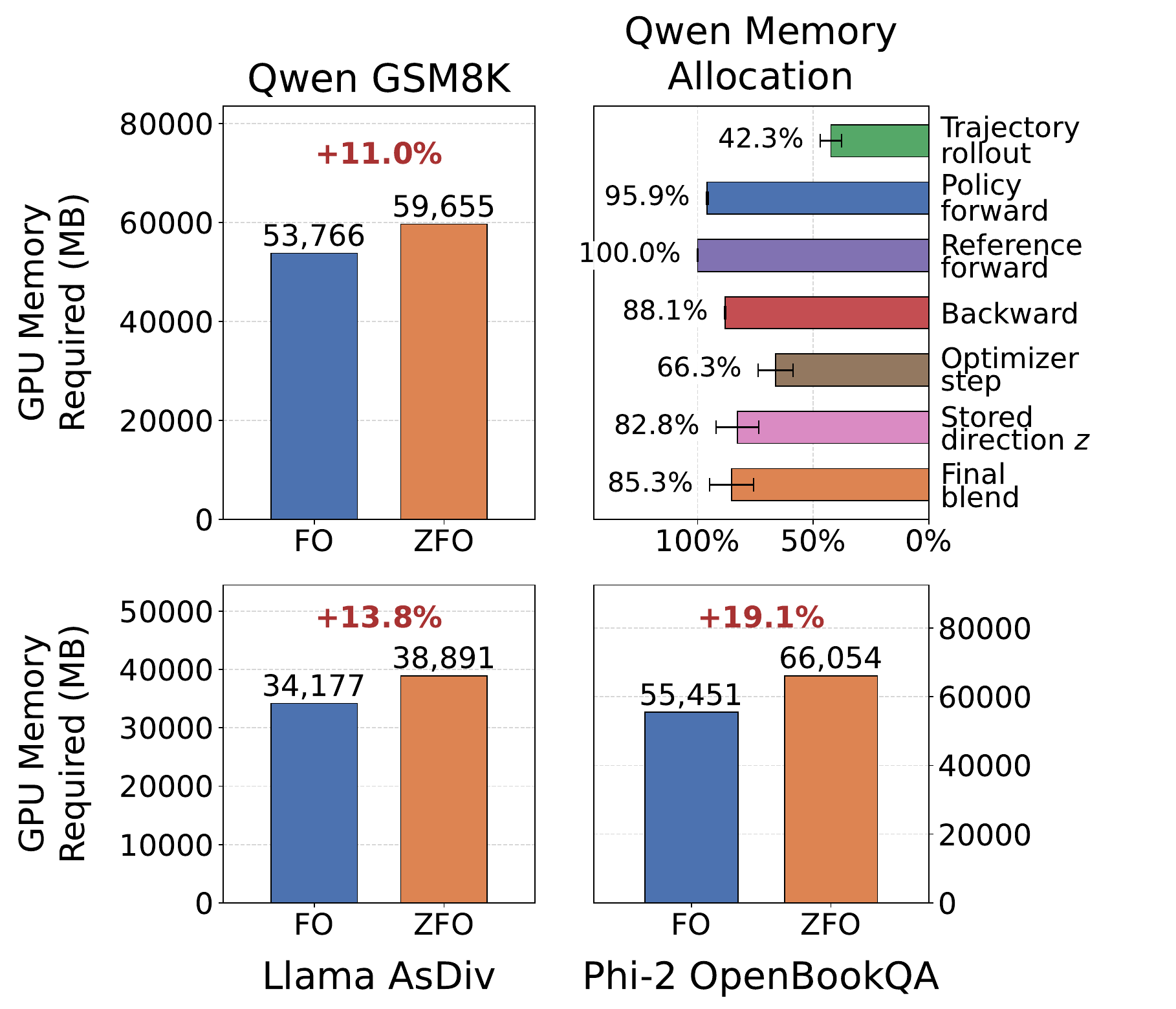}
        \captionof{figure}{\small{\textbf{GPU memory requirement comparisons.} Comparison of FO~FT and ZFO~FT across three configurations. The top right panel depicts the relative memory allocation for FO and ZFO FT.}}
        \label{fig:peak_memory}
    \end{minipage}
    \vspace{-4mm}
\end{figure}

Figure~\ref{fig:time_graph} decomposes per-iteration runtime for Qwen-2.5-MATH-1.5B on the benchmarks reported in Table~\ref{tab:empirics}. ZFO adds $\sim 5.2\%$ relative wall-clock overhead in this configuration. The corresponding GPU memory costs are examined below.



Figure~\ref{fig:peak_memory} reports peak GPU memory across three model–dataset pairs, alongside a per-phase breakdown for Qwen-2.5-MATH-1.5B. The overhead ranges from 11.0\%
 (Qwen/GSM8K) to 19.1\%
(Phi-2/OpenBookQA). We attribute this range to configuration differences; ZFO probes are absorbed with a smaller relative increase, whereas configurations with a small fixed group size and a larger batch see the overhead grow more. The memory allocation panel shows that the dominant cost remains the reference forward pass, while storing the update direction and applying the final blend sit below the existing bottleneck; hence, ZFO does not introduce a new peak; it widens the existing one. We note that the memory-saving trick MeZO relies on does not transfer here, since our update direction is the FO gradient rather than a random direction.


\label{subsec:cost}

%% file: Sections/06-conclusion.tex
\vspace{-2mm}
\section{Limitations, Discussion, and Conclusion}\label{sec:conclusion}
\vspace{-2mm}


\smartparagraph{What can two probes recover?} The backward pass already supplies the local slope; two symmetric probes add estimates of curvature and its variation. Theorem~\ref{thm:second-order-local-bound} translates local-model error into a bound on the value lost relative to the best step in $I$. A wider interval permits larger updates but increases approximation error. Bounds and Padé fallbacks limit extrapolation and avoid poles, but do not guarantee that every step improves the surrogate. Theorem~\ref{thm:zfo-stationarity} then connects step-selection accuracy to convergence. Under the stated assumptions, ZFO therefore retains a stationarity guarantee even with approximate steps.

\smartparagraph{Padé$_2$ Behavior.} \texttt{Padé$_2$} provides a useful diagnostic because, as shown in the toy validation in \S\ref{sec:perf_insights}, it behaves the most like FO since its updates select the endpoint of the search region unless a pole is detected, in which case it falls back to \texttt{Taylor$_2$}. This makes \texttt{Padé$_2$} a conservative reference for distinguishing whether ZFO gains come from expanding the search region alone or intelligent step selection. In Tables~\ref{tab:empirics} and~\ref{tab:qwen_substudy}, the richer surrogates frequently outperform \texttt{Padé$_2$} (e.g., Qwen on MATH, ARC-Challenge, and CODAH). Conversely, when \texttt{Padé$_2$} improves over FO, the expanded search region contributes in beneficial ways that are not obtainable simply by scaling $\alpha.$



\smartparagraph{Beyond the Four Instances.} We view \texttt{Taylor}$_2$, \texttt{Taylor}$_3$, \texttt{Padé}$_2$, and \texttt{Padé}$_3$ as initial ZFO instances rather than the strongest possible realizations of the framework. Each method is a simple analytic model built from lightweight probes, yet these lightweight probes already yield measurable gains in evaluation performance and training stability across several benchmarks.  
The main opportunity ZFO provides is not limited to the specific empirical gains reported in this paper, but instead lies in the observation that a meaningful curvature signal can be extracted along the FO direction at low additional cost.

\label{subsec:lim_and_disc}

\smartparagraph{Conclusion.} Step-size selection is a classic problem in large-scale optimization. The direction selected by a first-order optimizer is valuable, but the step length remains uncertain. We show that two additional function evaluations along that direction are sufficient to construct useful local models of the one-dimensional objective. This creates a general fixed-query step-selection framework, ZFO. In contrast to ZO methods, ZFO does not estimate a gradient with zeroth-order queries; it estimates how the objective changes along a gradient direction whose orientation has already been supplied by a first-order optimizer.

Existing methods have already explored function-based step selection along a first-order direction, most notably GeN. Our contribution elevates this idea from a single-step rule to a fixed-query directional modeling framework---the same two probes can support multiple local models, including higher-order Taylor and rational Padé approximations, while the analysis separates finite-difference estimation error, local-model approximation error, and global optimization error. Therefore, ZFO is not primarily a new Taylor or Padé step-size formula. It is a fixed-query framework for exploiting zeroth-order information along an already trusted first-order direction, together with a theory that explains the statistical and optimization error introduced by this separation. In that sense, GeN is a special case of the broader ZFO abstraction, rather than an unrelated baseline. In the same spirit, we encourage practitioners to add their favorite functional approximators to the ZFO framework. 

Across our experiments, ZFO often improves step selection over the fixed FO update. The quality of the selected step remains dependent on the local model, search radius, and alignment of the underlying first-order direction.

%% file: appendix.tex
\appendix
\input{Sections/02-related-work}

\section{Addendum to Theoretical Analysis}
\label{sec:add_theory}

This section provides the theoretical foundations needed to characterize when, why, and how ZFO can reliably improve step selection. We first establish concentration guarantees for the finite-difference curvature estimates, then bound how closely the resulting models recover the best step within their local interval, and we finally connect these local guarantees to global stationarity of the ZFO update. Together, these results separate the sources of error introduced by stochastic estimation, local function approximation, and step selection, providing a theoretical justification for the mechanisms used by ZFO.


    




To justify using finite-difference curvature information for ZFO step selection, we first ask whether the second- and third-order directional derivatives estimated from a shared batch can reliably approximate their population counterparts.

\smartparagraph{Hoeffding's Inequality.}
Let $X$ be a random variable such that $a\leq X\leq b$ almost surely. Then Hoeffding's lemma gives
\begin{equation}
\label{eq:hoeffding_subgaussian}
\mathbb{E}\left[
\exp\left(
\lambda(X-\mathbb{E}[X])
\right)
\right]
\leq
\exp\left(
\frac{\lambda^2(b-a)^2}{8}
\right),
\qquad {\rm for\;all\;} \lambda\in\mathbb{R}.
\end{equation}
Thus, every bounded random variable is sub-Gaussian with parameter $\sigma=(b-a)/2$.
    
    Let $X_1,\ldots,X_n$ be independent sub-Gaussian random variables with parameter $\sigma>0$, i.e.,
    \[
    \mathbb{E}\left[e^{\lambda X_i}\right]
    \leq
    e^{\lambda^2\sigma^2/2},
    \quad \forall \lambda\in\mathbb{R}.
    \]
    Then the sub-Gaussian concentration inequality gives:
    \begin{eqnarray}\label{eq:subgaussian_concentration}
    \mathbb{P}\left(
    \left|
    \frac{1}{n}\sum_{i=1}^n X_i
    \right|
    \leq
    \sigma\sqrt{\frac{2\log(2/\zeta)}{n}}
    \right)
    \geq
    1-\zeta.
    \end{eqnarray}





\setcounter{theorem}{0}
\begin{theorem}[CRN curvature concentration]
Suppose Assumptions~\ref{ass:pathwise-smoothness}--\ref{ass:subgaussian} hold. Then, with probability at least $1-\delta$,
$
\left|
\widehat{\phi''}(0)-\psi''(0)
\right|
\leq
\frac{B\epsilon^2}{12}
+
\sigma
\sqrt{
\frac{2\log(4/\delta)}{n}
},
$
and
$
\left|
\widehat{\phi'''}(0)-\psi'''(0)
\right|
\leq
\frac{L\epsilon^2}{20}
+
\sigma
\sqrt{
\frac{2\log(4/\delta)}{n}
}.
$
\end{theorem}
\begin{proof}
Fix an index $i\in\{1,\ldots,n\}$. By Assumption~\ref{ass:pathwise-smoothness} and the second-order central finite-difference expansion, there exists $\hat{\xi}_i\in(-\epsilon,\epsilon)$ such that
$$
\frac{\phi_i(\epsilon)-2\phi_i(0)+\phi_i(-\epsilon)}{\epsilon^2}
=
\phi_i^{\prime\prime}(0)
+
\frac{\epsilon^2}{12}\phi_i^{(4)}(\hat{\xi}_i).
$$
Using Assumption~\ref{ass:derivative-bounds}, the remainder satisfies
$$
\left|
\frac{\epsilon^2}{12}\phi_i^{(4)}(\hat{\xi}_i)
\right|
\leq
\frac{B\epsilon^2}{12}.
$$
Averaging the above identity over the shared batch gives
$$
\widehat{\phi^{\prime\prime}}(0)
=
\frac{1}{n}\sum_{i=1}^n \phi_i^{\prime\prime}(0)
+
\frac{1}{n}\sum_{i=1}^n
\frac{\epsilon^2}{12}\phi_i^{(4)}(\hat{\xi}_i).
$$
Therefore, by the triangle inequality,
$$
\left|
\widehat{\phi^{\prime\prime}}(0)-\psi^{\prime\prime}(0)
\right|
\leq
\left|
\frac{1}{n}\sum_{i=1}^n
\left(
\phi_i^{\prime\prime}(0)-\psi^{\prime\prime}(0)
\right)
\right|
+
\frac{B\epsilon^2}{12}.
$$

By Assumptions~\ref{ass:crn} and~\ref{ass:subgaussian}, the random variables
$
\phi_i^{\prime\prime}(0)-\psi^{\prime\prime}(0),
\quad i=1,\ldots,n,
$
are independent, centered, and sub-Gaussian with parameter $\sigma$.  Applying Inequality~\ref{eq:subgaussian_concentration} with $\zeta=\delta/2$ yields, with probability at least $1-\delta/2$,
$$
\left|
\frac{1}{n}\sum_{i=1}^n
\left(
\phi_i^{\prime\prime}(0)-\psi^{\prime\prime}(0)
\right)
\right|
\leq
\sigma
\sqrt{
\frac{2\log(4/\delta)}{n}
}.
$$
Consequently,
$$
\left|
\widehat{\phi^{\prime\prime}}(0)-\psi^{\prime\prime}(0)
\right|
\leq
\frac{B\epsilon^2}{12}
+
\sigma
\sqrt{
\frac{2\log(4/\delta)}{n}
}.
$$

We now repeat the same argument for the third derivative estimate. By the third-order finite-difference expansion, for each $i$ there exists $\tilde{\xi}_i\in(-\epsilon,\epsilon)$ such that
$$
\frac{
3\left(
\phi_i(\epsilon)-\phi_i(-\epsilon)-2\epsilon\phi_i^{\prime}(0)
\right)
}
{\epsilon^3}
=
\phi_i^{\prime\prime\prime}(0)
+
\frac{\epsilon^2}{20}\phi_i^{(5)}(\tilde{\xi}_i).
$$
Assumption~\ref{ass:derivative-bounds} gives
$$
\left|
\frac{\epsilon^2}{20}\phi_i^{(5)}(\tilde{\xi}_i)
\right|
\leq
\frac{L\epsilon^2}{20}.
$$
Averaging over the shared batch and applying the triangle inequality gives
$$
\left|
\widehat{\phi^{\prime\prime\prime}}(0)-\psi^{\prime\prime\prime}(0)
\right|
\leq
\left|
\frac{1}{n}\sum_{i=1}^n
\left(
\phi_i^{\prime\prime\prime}(0)-\psi^{\prime\prime\prime}(0)
\right)
\right|
+
\frac{L\epsilon^2}{20}.
$$
Again, by Assumptions~\ref{ass:crn} and~\ref{ass:subgaussian}, the centered variables
$
\phi_i^{\prime\prime\prime}(0)-\psi^{\prime\prime\prime}(0),
\qquad i=1,\ldots,n,
$
are independent and sub-Gaussian with parameter $\sigma$. Applying Inequality~\ref{eq:subgaussian_concentration} with $\zeta=\delta/2$ yields, with probability at least $1-\delta/2$,
$$
\left|
\frac{1}{n}\sum_{i=1}^n
\left(
\phi_i^{\prime\prime\prime}(0)-\psi^{\prime\prime\prime}(0)
\right)
\right|
\leq
\sigma
\sqrt{
\frac{2\log(4/\delta)}{n}
}.
$$
Thus,
$$
\left|
\widehat{\phi^{\prime\prime\prime}}(0)-\psi^{\prime\prime\prime}(0)
\right|
\leq
\frac{L\epsilon^2}{20}
+
\sigma
\sqrt{
\frac{2\log(4/\delta)}{n}
}.
$$

Finally, taking a union bound over the two concentration events shows that both inequalities hold simultaneously with probability at least $1-\delta$.  Therefore,
$$
\left|
\widehat{\phi^{\prime\prime}}(0)-\psi^{\prime\prime}(0)
\right|
\leq
\frac{B\epsilon^2}{12}
+
\sigma
\sqrt{
\frac{2\log(4/\delta)}{n}
},
$$
and
$$
\left|
\widehat{\phi^{\prime\prime\prime}}(0)-\psi^{\prime\prime\prime}(0)
\right|
\leq
\frac{L\epsilon^2}{20}
+
\sigma
\sqrt{
\frac{2\log(4/\delta)}{n}
}.
$$
This completes the proof.
\end{proof}

\subsection{ZFO Maximum Error}
\label{subsec:zfo_analysis}



\begin{theorem}[Second-order Taylor approximation bound]
Suppose Assumptions~\ref{ass:pathwise-smoothness} and~\ref{ass:derivative-bounds} hold. Then the step selected by maximizing the second-order Taylor model satisfies
$
\phi(s_M)
\geq
\phi(s_{\phi})
-
\frac{M R^3}{3}
-
\frac{B\epsilon^2 R^2}{12}.
$
\end{theorem}

\begin{proof}
By Assumption~\ref{ass:pathwise-smoothness} and Taylor's theorem, there exists $\xi_s$ such that
$$
\phi(s)
=
\phi(0)+s\phi'(0)+s^2\frac{\phi''(0)}{2}+s^3\frac{\phi^{(3)}(\xi_s)}{6}.
$$
Similarly, for some $\hat{\xi}_s$, the approximation $T_2$ admits the expansion, 
$$
T_2(s)
=
\phi(0)+s\phi'(0)+s^2\frac{\phi''(0)}{2}+\epsilon^2s^2\frac{\phi^{(4)}(\hat{\xi}_s)}{24}.
$$
Subtracting these identities gives
$$
\phi(s)-T_2(s)
=
s^3\frac{\phi^{(3)}(\xi_s)}{6}-\epsilon^2s^2\frac{\phi^{(4)}(\hat{\xi}_s)}{24}.
$$
Therefore, by the triangle inequality and Assumption~\ref{ass:derivative-bounds} we have,
$$
|\phi(s)-T_2(s)|
\leq
\left|s^3\frac{\phi^{(3)}(\xi_s)}{6}\right|
+
\left|\epsilon^2s^2\frac{\phi^{(4)}(\hat{\xi}_s)}{24}\right|
\leq
\left|s^3\frac{M}{6}\right|
+
\left|\epsilon^2s^2\frac{B}{24}\right|
\leq
\frac{M{R}^3}{6}+\frac{\epsilon^2B{R}^2}{24}.
$$
Consequently,
$$
|\phi(s)-T_2(s)|
\leq
\frac{M{R}^3}{6}+\frac{\epsilon^2B{R}^2}{24}
\eqqcolon
\varphi.
$$
We have
$$
s_\phi
\coloneqq
\arg\max_{s\in[0,{R}]}\phi(s).
$$
By the definitions of $s_\phi$ and $s_M$,
$$
\phi(s_\phi)\geq\phi(s_M)
\qquad\text{and}\qquad
T_2(s_M)\geq T_2(s_\phi).
$$
Chaining these inequalities with the uniform bound $\varphi$ yields
$$
\phi(s_M)
\geq
T_2(s_M)-\varphi
\geq
T_2(s_\phi)-\varphi
\geq
\phi(s_\phi)-2\varphi.
$$
Therefore,
$$
\phi(s_M)
\geq
\phi(s_\phi)-\frac{M{R}^3}{3}-\frac{\epsilon^2B{R}^2}{12}.
$$
\end{proof}

\smartparagraph{Additional Model Error Bounds.} To understand how accurately ZFO’s local models can choose a step, we next ask how much objective value may be lost when maximizing a finite-difference Taylor or Padé approximation. The proof technique used in Theorem~\ref{thm:second-order-local-bound} generalizes naturally to higher-order ZFO models. This allows us to characterize the approximation gap of \texttt{Taylor$_3$} and \texttt{Pad\'e} over the search region $[0, {R}].$ Corollary~\ref{cor:taylor3-bound} establishes a similar bound for \texttt{Taylor$_3$} as Theorem~\ref{thm:second-order-local-bound} did for \texttt{Taylor$_2$}. Corollaries~\ref{cor:pade2_residual_bound} and \ref{cor:pade3_residual_bound} adapt a similar structure through the residual $r(s)\coloneqq D(s)\phi(s)-N(s).$ All three results yield finite-difference and model truncation errors.

\paragraph{Alternative Pad\'e3 Formulation.}
The third-order Pad\'e approximation is not restricted to the
$[1/2]$ formulation considered in our empirical experiments ($P_3(s)$).
An alternative $[2/1]$ formulation allocates the additional
polynomial degree to the numerator rather than the denominator,
yielding
\[
P_{[2,1]}(s)
=
\phi(0)
+
\frac{as+\rho s^2}{1+qs},
\qquad
q=-\frac{c}{3b},
\qquad
\rho=\frac{b}{2}+aq.
\]
Both formulations recover the same third-order local Taylor
expansion, provided their coefficients are well-defined, but
differ in their extrapolation behavior and pole structure.
The $[2/1]$ formulation therefore provides another possible
realization of Pad\'e3 within the ZFO framework.


\begin{definition}
\label{def:pade_residual}
Let $P_{k}(s)$ be the $k$-th order Pad\'e approximation of $\phi(s)$ over 
$[0,{R}]$, where
$
P_{k}(s)=\frac{N(s)}{D(s)}, 
N(s)=a_0+a_1s,
$
with $D(s)=1+d_1s$ for $k=2$ and $D(s)=1+d_1s+d_2s^2$ for $k=3$,
and ${R}=s_{\rm FO}\beta,$ such that $\beta>0$. The coefficients of 
$N(s)$ and $D(s)$ are chosen so that $P_{k}(s)$ agrees with the $k$-th order Taylor model $T_{k}(s)$ when Taylor expanded. 
Define the Pad\'e residual
$
r(s):=D(s)\phi(s)-N(s).
$
\end{definition}

When we use the Pad\'e$_2$ and Pad\'e$_3$ models we assume the following:
\begin{assumption}
\label{ass:pade_residual_regular}
The Pad\'e denominator is bounded away from zero on $[0,{R}]$, i.e.,
there exists $\eta>0$ such that
$
|D(s)|\geq\eta, \forall s\in[0,{R}],
$
for both Pad\'e$_2$ and Pad\'e$_3$.
\end{assumption}

\begin{assumption}
\label{ass:bounded_residual}
The Pad\'e residual has bounded higher-order derivatives, i.e., there exist 
$Q,P>0$ such that
$
|r^{(3)}(s)|\leq Q
\quad\text{(Pad\'e$_2$)},
\quad
|r^{(4)}(s)|\leq P
\quad\text{(Pad\'e$_3$)},
\quad \forall s\in[0,{R}].
$
\end{assumption}

Assumption~\ref{ass:pade_residual_regular} is consistent with our safeguard of \texttt{Pad\'e$_3$} falling back to \texttt{Taylor$_2$} when encountering a pole.

\begin{corollary}[Third-order Taylor approximation bound]
\label{cor:taylor3-bound}
Suppose Assumptions \ref{ass:pathwise-smoothness}-\ref{ass:derivative-bounds} hold. Then the step selected by maximizing the third-order Taylor model satisfies $
\phi(s_M) \ge \phi(s_\phi) - \frac{B R^4}{12} - \frac{B \epsilon^2 R^2}{12} - \frac{L \epsilon^2 R^3}{60}.$
\end{corollary}

\begin{proof}
By Assumption~\ref{ass:pathwise-smoothness} and Taylor's theorem, there exists $\xi_s$ such that
$$
\phi(s)
=
\phi(0)+s\phi'(0)+s^2\frac{\phi''(0)}{2}+s^3\frac{\phi'''(0)}{6}+s^4\frac{\phi^{(4)}(\xi_s)}{24}.
$$
Similarly, the approximation $T_3$ admits the expansion, for some $\hat{\xi}_s$ and $\tilde{\xi}_s$,
$$
T_3(s)
=
\phi(0)+s\phi'(0)+s^2\frac{\phi''(0)}{2}+s^3\frac{\phi'''(0)}{6}+\epsilon^2s^2\frac{\phi^{(4)}(\hat{\xi}_s)}{24}+\epsilon^2s^3\frac{\phi^{(5)}(\tilde{\xi}_s)}{120}.
$$
Subtracting these identities gives
$$
\phi(s)-T_3(s)
=
s^4\frac{\phi^{(4)}(\xi_s)}{24}-\epsilon^2s^2\frac{\phi^{(4)}(\hat{\xi}_s)}{24}-\epsilon^2s^3\frac{\phi^{(5)}(\tilde{\xi}_s)}{120}.
$$
Therefore, by the triangle inequality and Assumption~\ref{ass:derivative-bounds},
$$
|\phi(s)-T_3(s)|
\leq
\left|s^4\frac{\phi^{(4)}(\xi_s)}{24}\right|
+
\left|\epsilon^2s^2\frac{\phi^{(4)}(\hat{\xi}_s)}{24}\right|
+
\left|\epsilon^2s^3\frac{\phi^{(5)}(\tilde{\xi}_s)}{120}\right|
\leq
\left|s^4\frac{B}{24}\right|
+
\left|\epsilon^2s^2\frac{B}{24}\right|
+
\left|\epsilon^2s^3\frac{L}{120}\right|
$$ 
$$
\leq
\frac{B{R}^4}{24}+\frac{\epsilon^2B{R}^2}{24}+\frac{\epsilon^2L{R}^3}{120}.
$$
Consequently,
$$
|\phi(s)-T_3(s)|
\leq
\frac{B{R}^4}{24}+\frac{\epsilon^2B{R}^2}{24}+\frac{\epsilon^2L{R}^3}{120}
\eqqcolon
\varphi.
$$
Define
$$
s_\phi
\coloneqq
\arg\max_{s\in[0,{R}]}\phi(s).
$$
By the definitions of $s_\phi$ and $s_M$,
$$
\phi(s_\phi)\geq\phi(s_M)
\qquad\text{and}\qquad
T_3(s_M)\geq T_3(s_\phi).
$$
Chaining these inequalities with the uniform bound $\varphi$ yields
$$
\phi(s_M)
\geq
T_3(s_M)-\varphi
\geq
T_3(s_\phi)-\varphi
\geq
\phi(s_\phi)-2\varphi.
$$
Therefore,
$$
\phi(s_M)
\geq
\phi(s_\phi)-\frac{B{R}^4}{12}-\frac{\epsilon^2B{R}^2}{12}-\frac{\epsilon^2L{R}^3}{60}.
$$
\end{proof}

\begin{corollary}[Pad\'e$_2$ approximation bound]
\label{cor:pade2_residual_bound}
Let Assumptions~\ref{ass:pathwise-smoothness}-\ref{ass:derivative-bounds} and \ref{ass:pade_residual_regular}-\ref{ass:bounded_residual} hold. Then the step selected by maximizing the \texttt{Pad\'e$_2$} model satisfies
$
\phi(s_M)
\geq
\phi(s_\phi)
-
\frac{B\epsilon^2{R}^2}{12\eta}
-
\frac{Q{R}^3}{3\eta}.
$
\end{corollary}

\begin{proof}
Since $P_{2}(s)=N(s)/D(s)$, we have
$$
\phi(s)-P_{2}(s)
=
\frac{D(s)\phi(s)-N(s)}{D(s)}
=
\frac{r(s)}{D(s)}.
$$
By construction, $P_2(s)$ agrees with the second-order Taylor model $T_2(s)$ when Taylor expanded. Since $\phi(0)$ and $\phi'(0)$ are exact, this gives
$$
r(0)=0,\qquad r'(0)=0.
$$
Notice that
$$
r(s)
=
D(s)(\phi(s)-P_2(s))
=
(1+d_1s)
\left(
\frac{\phi''(0)-\widehat{\phi''}(0)}{2}s^2
+
\mathcal{O}(s^3)
\right).
$$
Therefore, the second-order coefficient satisfies
$$
\frac{r''(0)}{2}
=
\frac{\phi''(0)-\widehat{\phi''}(0)}{2},
$$
and using the finite-difference error from Lemma~\ref{lemma:second-order-fin-dif} yields
$$
\left|\frac{r''(0)}{2}\right|
\leq
\frac{B\epsilon^2}{24}.
$$
Taylor expanding $r(s)$ gives
$$
r(s)
=
\frac{s^2}{2}r''(0)
+
\frac{s^3}{6}r^{(3)}(\xi_s)
$$
for some $\xi_s\in[0,s]$. Therefore, by Assumption~\ref{ass:bounded_residual},
$$
|r(s)|
\leq
\frac{B\epsilon^2 s^2}{24}
+
\frac{Q s^3}{6}.
$$
Since $|D(s)|\geq\eta$ by Assumption~\ref{ass:pade_residual_regular},
$$
|\phi(s)-P_2(s)|
\leq
\frac{1}{\eta}
\left[
\frac{B\epsilon^2{R}^2}{24}
+
\frac{Q{R}^3}{6}
\right]
\eqqcolon
\varphi.
$$
By the definition of $s_M=\arg\max_{s\in[0,{R}]}P_2(s)$,
$$
P_2(s_M)\geq P_2(s_\phi).
$$
Chaining this inequality with the uniform bound $\varphi$ yields
$$
\phi(s_M)
\geq
P_2(s_M)-\varphi
\geq
P_2(s_\phi)-\varphi
\geq
\phi(s_\phi)-2\varphi.
$$
Substituting the value of $\varphi$ gives
$$
\phi(s_M)
\geq
\phi(s_\phi)
-
\frac{B\epsilon^2{R}^2}{12\eta}
-
\frac{Q{R}^3}{3\eta}.
$$
Hence, the result.
\end{proof}

\begin{corollary}[Pad\'e$_3$ approximation bound]
\label{cor:pade3_residual_bound}

Let Assumptions~\ref{ass:pathwise-smoothness}-\ref{ass:derivative-bounds} and \ref{ass:pade_residual_regular}-\ref{ass:bounded_residual} hold. Then the step selected by maximizing the \texttt{Pad\'e$_3$} model satisfies $
\phi(s_M)
\geq
\phi(s_\phi)
-
\frac{B\epsilon^2{R}^2}{12\eta}
-
\frac{L\epsilon^2{R}^3}{60\eta}
-
\frac{|d_1|B\epsilon^2{R}^3}{12\eta}
-
\frac{P{R}^4}{12\eta}.
$
\end{corollary}

\begin{proof}
Since $P_{3}(s)=N(s)/D(s)$, we have
$$
\phi(s)-P_{3}(s)
=
\frac{D(s)\phi(s)-N(s)}{D(s)}
=
\frac{r(s)}{D(s)}.
$$
By construction, $P_3(s)$ agrees with the third-order Taylor model $T_3(s)$ when Taylor expanded. Since $\phi(0)$ and $\phi'(0)$ are exact, this gives
$$
r(0)=0,\qquad r'(0)=0.
$$
Notice that
$$
r(s)
=
D(s)(\phi(s)-p_3(s))
=
(1+d_1s+d_2s^2)
\left(
\frac{\phi''(0)-\hat{\phi}''(0)}{2}s^2
+
\frac{\phi'''(0)-\hat{\phi}'''(0)}{6}s^3
+
\mathcal{O}(s^4)
\right).
$$
Therefore, the second-order coefficient satisfies
$$
\frac{r''(0)}{2}
=
\frac{\phi''(0)-\hat{\phi}''(0)}{2},
$$
and using the finite-difference error yields
$$
\left|\frac{r''(0)}{2}\right|
\leq
\frac{B\epsilon^2}{24}.
$$
Similarly, the third-order coefficient satisfies
$$
\frac{r'''(0)}{6}
=
\frac{\phi'''(0)-\hat{\phi}'''(0)}{6}
+
d_1\frac{\phi''(0)-\hat{\phi}''(0)}{2}.
$$
Consequently,
$$
\left|\frac{r'''(0)}{6}\right|
\leq
\frac{L\epsilon^2}{120}
+
|d_1|\frac{B\epsilon^2}{24}.
$$
Taylor expanding $r(s)$ gives
$$
r(s)
=
\frac{s^2}{2}r''(0)
+
\frac{s^3}{6}r'''(0)
+
\frac{s^4}{24}r^{(4)}(\xi_s)
$$
for some $\xi_s\in[0,s]$. Therefore, by Assumption~\ref{ass:bounded_residual},
$$
|r(s)|
\leq
\frac{B\epsilon^2s^2}{24}
+
\left(
\frac{L\epsilon^2}{120}
+
|d_1|\frac{B\epsilon^2}{24}
\right)s^3
+
\frac{Ps^4}{24}.
$$
Since $|D(s)|\geq\eta$ by Assumption~\ref{ass:pade_residual_regular},
$$
|\phi(s)-P_3(s)|
\leq
\frac{1}{\eta}
\left[
\frac{B\epsilon^2{R}^2}{24}
+
\left(
\frac{L\epsilon^2}{120}
+
|d_1|\frac{B\epsilon^2}{24}
\right){R}^3
+
\frac{P{R}^4}{24}
\right]
\eqqcolon
\varphi.
$$
By the definition of $s_M=\arg\max_{s\in[0,{R}]}P_3(s)$,
$$
P_3(s_M)\geq P_3(s_\phi).
$$
Chaining this inequality with the uniform bound $\varphi$ yields
$$
\phi(s_M)
\geq
P_3(s_M)-\varphi
\geq
P_3(s_\phi)-\varphi
\geq
\phi(s_\phi)-2\varphi.
$$
Substituting the value of $\varphi$ gives
$$
\phi(s_M)
\geq
\phi(s_\phi)
-
\frac{B\epsilon^2{R}^2}{12\eta}
-
\frac{L\epsilon^2{R}^3}{60\eta}
-
\frac{|d_1|B\epsilon^2{R}^3}{12\eta}
-
\frac{P{R}^4}{12\eta}.
$$
\end{proof}

Together, these three corollaries and Theorem~\ref{thm:second-order-local-bound} quantify how the additional curvature information used by ZFO models affects the worst-case approximation gap. The \texttt{Taylor$_3$} bound replaces the $\mathcal{O}({R}^3)$ term in Theorem~\ref{thm:second-order-local-bound} with an $\mathcal{O}({R}^4)$ term. Hence the relationship between the errors of \texttt{Taylor$_2$} and \texttt{Taylor$_3$} can be most significantly characterized by the power of their search regions ${R}.$ The \texttt{Pad\'e$_2$} bound remains closer in order to the \texttt{Taylor$_2$} bound, but introduces the residual-controlled $\frac{Q R^3}{3\eta}$ term. The \texttt{Pad\'e$_3$} bound exhibits a similar $\mathcal{O}({R}^4)$ term as \texttt{Taylor$_3$} and also uses a residual error term. Overall, the two corollaries reinforce the same conclusion as Theorem~\ref{thm:second-order-local-bound}; when the search region is moderate and the perturbation scale is small, then ZFO maximizers achieve values close to the best step within a local trust region.

\smartparagraph{ZFO Dominance.} As a supplementary result beyond the main-text analysis, we next ask whether the ZFO-selected \texttt{Taylor$_2$} step can be guaranteed to perform at least as well as the original first-order step under suitable local curvature conditions.


Under Assumption~\ref{ass:pathwise-smoothness} and Assumptions~\ref{ass:f_difference_bound}-\ref{ass:interior_maximum}, we compare the \texttt{Taylor2} step with the typical FO step by bounding the curvature approximation error, imposing strong concavity on $[0,{R}]$, and requiring the Taylor-model maximizer to be an interior point. The theorem gives conditions to guarantee $\phi(s_M)\geq \phi(s_{\rm FO})$.

\begin{assumption}
\label{ass:f_difference_bound}
    Let the difference between the second derivative of $\phi$ and the second-order approximation of $\phi$ at $s=0$ be bounded, i.e., there exists $D>0$ such that $|\phi''(s)-\frac{\phi(\epsilon) - 2\phi(0)+\phi(-\epsilon)}{\epsilon^2}| \leq D; \forall s \in [0,{R}].$
\end{assumption}

\begin{assumption}
\label{ass:f_strong_concavity}
    Let $\phi$ be $\nu$-strongly concave on $[0,{R}]$, i.e.,
    $\phi(u) \leq \phi(s) + \phi'(s)(u-s) - \frac{\nu}{2}(u-s)^2; $  for all $ s,u \in [0,{R}]$ where $\nu > 0.$
\end{assumption}

\begin{assumption}
\label{ass:interior_maximum}
    Let the maximum of $T_2(s)$ be an interior point, i.e., $-\frac{\epsilon^2\phi'(0)}{\phi(\epsilon)-2\phi(0)+\phi(-\epsilon)} \in [0, {R}]$ and $\phi(\epsilon)-2\phi(0)+\phi(-\epsilon) < 0$.
\end{assumption}

\setcounter{theorem}{3}

\begin{theorem}[Second-order Taylor win guarantee]\label{thm:second-order-winner} 
If Assumption~\ref{ass:pathwise-smoothness} and Assumptions~\ref{ass:f_difference_bound}-\ref{ass:interior_maximum} hold, and if $s_{\rm FO}\in[0,{R}]$, then the following relation holds between $\phi(s_M)$ and  $\phi(s_{\rm FO})$:
$$
\text{if } |s_{\rm FO}-s_M|=0
\text{ or }
|s_{\rm FO}-s_M|\geq\frac{2Ds_M}{\nu},
\text{ then }
\phi(s_M)\geq\phi(s_{\rm FO}).
$$
\end{theorem}

\begin{proof}
We have
$$
T_2(s)=\phi(0)+\phi'(0)s+\frac{1}{2}\widehat{\phi''}(0)s^2,
$$
which gives
$$
T_2'(s)=\phi'(0)+\widehat{\phi''}(0)s.
$$
By Assumption~\ref{ass:interior_maximum}, there exists $s_M\in[0,{R}]$ such that
$$
T_2'(s_M)=0,
$$
which gives
$$
\phi'(0)+\widehat{\phi''}(0)s_M=0,
$$
and therefore
$$
s_M=-\frac{\phi'(0)}{\widehat{\phi''}(0)}.
$$
Notice that
$$
\phi'(s)-T_2'(s)
=
\phi'(s)-\phi'(0)-\widehat{\phi''}(0)s
=
\int_{0}^{s}\left(\phi''(u)-\widehat{\phi''}(0)\right)du.
$$
Therefore, by Assumption~\ref{ass:f_difference_bound},
$$
|\phi'(s_M)-T_2'(s_M)|
=
\left|\int_{0}^{s_M}\left(\phi''(u)-\widehat{\phi''}(0)\right)du\right|
\leq
\int_{0}^{s_M}\left|\phi''(u)-\widehat{\phi''}(0)\right|du
\leq
\int_{0}^{s_M}Ddu
=
Ds_M.
$$
Since $T_2'(s_M)=0$, this gives
$$
|\phi'(s_M)|
\leq
Ds_M.
\qquad(\star)
$$
Let
$$
d=s_{\rm FO}-s_M.
$$
If $|d|=0$, then $s_{\rm FO}=s_M$, and hence
$$
\phi(s_M)=\phi(s_{\rm FO}).
$$
Now suppose $|d|>0$. By Assumption~\ref{ass:f_strong_concavity},
$$
\phi(s_{\rm FO})
\leq
\phi(s_M)+\phi'(s_M)d-\frac{\nu}{2}d^2,
$$
which gives
$$
\phi(s_{\rm FO})-\phi(s_M)
\leq
\phi'(s_M)d-\frac{\nu}{2}d^2.
$$
Furthermore, by $(\star)$,
$$
\phi'(s_M)d
\leq
|\phi'(s_M)||d|
\leq
Ds_M|d|.
$$
Therefore,
$$
\phi(s_{\rm FO})-\phi(s_M)
\leq
Ds_M|d|-\frac{\nu}{2}d^2
=
|d|\left(Ds_M-\frac{\nu}{2}|d|\right).
$$
Thus, if
$$
|d|\geq\frac{2Ds_M}{\nu},
$$
then
$$
Ds_M-\frac{\nu}{2}|d|\leq0,
$$
and therefore
$$
\phi(s_{\rm FO})-\phi(s_M)\leq0.
$$
Hence
$$
\phi(s_M)\geq\phi(s_{\rm FO}).
$$
Finally, since $d=s_{\rm FO}-s_M$, the sufficient condition is
$$
|s_{\rm FO}-s_M|\geq\frac{2Ds_M}{\nu}.
$$
If $s_{\rm FO}=\alpha\phi'(0)$, then
$$
|s_{\rm FO}-s_M|
=
\left|\alpha\phi'(0)-\left(-\frac{\phi'(0)}{\widehat{\phi''}(0)}\right)\right|
=
|\phi'(0)|\left|\alpha+\frac{1}{\widehat{\phi''}(0)}\right|.
$$
Using $s_M=-\frac{\phi'(0)}{\widehat{\phi''}(0)}$, this gives the equivalent condition
$$
\left|\alpha+\frac{1}{\widehat{\phi''}(0)}\right|
\geq
-\frac{2D}{\nu\widehat{\phi''}(0)}.
$$
Therefore, if $|s_{\rm FO}-s_M|=0$ or $|s_{\rm FO}-s_M|\geq\frac{2Ds_M}{\nu}$, then
$
\phi(s_M)\geq\phi(s_{\rm FO}).
$
\end{proof}


\subsection{ZFO Global Stationarity}

\label{subsec:global_convergence}

To connect the local ZFO step-selection guarantees to overall optimization behavior, we finally ask whether repeatedly applying these approximate directional steps still drives the training process toward a stationary point of the population objective.

ZFO is computed on the sampled surrogate $F_t(\theta)=F(\theta;S_t)$, where $S_t$ is generated at
the current iterate $\theta_t$ and then fixed for nearby perturbations.  
We define the expected surrogate obtained under the current sampling distribution $\bar F_{t}(\theta)\coloneqq\mathbb E_{S}[F(\theta;S)]$. In our theoretical analysis, the true objective $J$ and the expected surrogate $\bar F_t$ have the same gradient
$\nabla \bar F_{t}(\theta)=\nabla J(\theta_t)$.
In this formulation $F_t$ is used for parameter updates, while its expectation $\bar F_{t}$ develops a reliable objective signal. 

\setcounter{lemma}{0}
\begin{lemma}[Surrogate vs Empirical Difference]
    Let Assumptions~\ref{ass:pathwise-smoothness}-\ref{ass:smooth-bounded} hold, then 
    
    $\max_{s\in I_t}
J(\theta_t+s z_t)
-
J(\theta_t+s_t z_t) \leq \Delta_t,$ holds with probability at least $1- \delta,$ where  $$\Delta_t \coloneq (\mathcal{L}_J+\mathcal{L}_{\bar F})R_t^2+2\sigma(1+R_t)\sqrt{\frac{2\log(4/\delta)}{n}}+2WR_t^2+ \phi(s_{\phi}) - \phi(s_M).
$$ 
\end{lemma}

\begin{proof}
Let $j_t(s)\coloneqq J(\theta_t+sz_t)$, so that
$$
\max_{s\in I_t}J(\theta_t+sz_t)-J(\theta_t+s_Mz_t)
=j_t(s^*)-j_t(s_M),
$$
where $s^*$ maximizes $j_t(s)$ over $s\in I_t$.

Here $\psi_t(s)=\bar F_t(\theta_t+sz_t)$, where $\bar F_t$ is defined before substituting the batch-selected direction. Since $\nabla J(\theta_t)=\nabla\bar{F}_t(\theta_t)$, we have $j_t'(0)=\psi_t'(0)$. Therefore, for any $s\in I_t$,
$$
j_t(s)-j_t(0)
=
sj_t'(0)+\int_{0}^{s}(s-u)j_t''(u)du.
$$
Similarly,
$$
\psi_t(s)-\psi_t(0)
=
s\psi_t'(0)+\int_{0}^{s}(s-u)\psi_t''(u)du.
$$
Subtracting these identities gives
$$
[j_t(s)-j_t(0)]-[\psi_t(s)-\psi_t(0)]
=
\int_{0}^{s}(s-u)[j_t''(u)-\psi_t''(u)]du.
$$
By Assumption~\ref{ass:smooth-bounded}, $|j_t''(u)|\leq\mathcal{L}_J$ and $|\psi_t''(u)|\leq\mathcal{L}_{\bar{F}}$ almost everywhere, and therefore, by the triangle inequality,
$$
[j_t(s)-j_t(0)]-[\psi_t(s)-\psi_t(0)]
\leq
(\mathcal{L}_J+\mathcal{L}_{\bar{F}})\int_{0}^{s}(s-u)du
=
\frac{\mathcal{L}_J+\mathcal{L}_{\bar{F}}}{2}s^2.
$$
Consequently,
$$
j_t(s^*)-j_t(0)
\leq
\frac{\mathcal{L}_J+\mathcal{L}_{\bar{F}}}{2}(s^*)^2+\psi_t(s^*)-\psi_t(0),
$$
and similarly,
$$
j_t(s_M)-j_t(0)
\geq
\psi_t(s_M)-\psi_t(0)-\frac{\mathcal{L}_J+\mathcal{L}_{\bar{F}}}{2}(s_M)^2.
$$
Combining these inequalities yields
$$
\max_{s\in I_t}J(\theta_t+sz_t)-J(\theta_t+s_Mz_t)
\leq
\psi_t(s^*)-\psi_t(s_M)+(\mathcal{L}_J+\mathcal{L}_{\bar{F}})R_t^2.
$$
We now bound the difference $\psi_t(s^*)-\psi_t(s_M)$. By the triangle inequality,
$$
\psi_t(s^*)-\psi_t(s_M)
\leq
|\psi_t(s^*)-\phi_t(s^*)|+\phi_t(s^*)-\phi_t(s_M)+|\phi_t(s_M)-\psi_t(s_M)|.
$$
By the definition of $s_\phi$,
$$
\phi_t(s^*)-\phi_t(s_M)\leq\phi_t(s_\phi)-\phi_t(s_M).
$$
Furthermore,
$$
|\psi_t(s^*)-\phi_t(s^*)|+|\phi_t(s_M)-\psi_t(s_M)|
\leq
2\max_{s\in I_t}|\psi_t(s)-\phi_t(s)|.
$$
To bound $\max_{s\in I_t}|\psi_t(s)-\phi_t(s)|$, Taylor expanding gives
$$
|\psi_t(s)-\phi_t(s)|
=
\left|(\psi_t-\phi_t)(0)+s(\psi_t'-\phi_t')(0)+\int_0^s(s-u)(\psi_t''-\phi_t'')(u)du\right|
$$
$$
\leq
|(\psi_t-\phi_t)(0)|+|(\psi_t'-\phi_t')(0)|R_t+WR_t^2,
$$
where we take $W\geq\mathcal L_{\bar F}$ without loss of generality.
Therefore, by the conditional concentration premise,
$$
\max_{s\in I_t}|\psi_t(s)-\phi_t(s)|
\leq
\sigma(1+R_t)\sqrt{\frac{2\log(4/\delta)}{n}}+WR_t^2.
$$
Consequently, with conditional probability at least $1-\delta$, $j_t(s^*)-j_t(s_M)\leq\Delta_t$, where
$$
\Delta_t
\coloneq
(\mathcal{L}_J+\mathcal{L}_{\bar{F}})R_t^2
+
2\sigma(1+R_t)\sqrt{\frac{2\log(4/\delta)}{n}}
+
2WR_t^2
+
(\phi_t(s_\phi)-\phi_t(s_M)).
$$
\end{proof}

\setcounter{theorem}{2}

\begin{theorem}[Stationarity of the ZFO update]
Let Assumptions~\ref{ass:pathwise-smoothness}--\ref{ass:alignment} hold. Let $\mathcal K\coloneq\mathcal L_J+\mathcal L_{\bar F}+2W$ and $\chi\coloneq 3\mu^2-4\delta>0$, with $q_t/(2\mathcal K)\in I_t$ almost surely. Then, for every $T>0$, we have
$$
\min_{0\leq t<T}
\mathbb{E}\|\nabla J(\theta_t)\|^2
\leq
\frac{
8\mathcal K(J^\star-J(\theta_0))
}
{\chi T}+
\frac{8\mathcal K \tilde \Delta_T}{\chi}, \text{ where } \tilde \Delta_T=\frac{1}{T}\sum_{t=0}^{T-1} \mathbb{E}[\Delta_t].
$$
\end{theorem}

\begin{proof}
Since $\frac{q_t}{2\mathcal K}\in I_t$, the comparison step exists within the relevant search interval. Therefore,
$$
\max_{s\in I_t}J(\theta_t+sz_t)
\geq
J\left(\theta_t+\frac{q_t}{2\mathcal K}z_t\right).
$$
By Assumption~\ref{ass:smooth-bounded},
$$
J\left(\theta_t+\frac{q_t}{2\mathcal K}z_t\right)
\geq
J(\theta_t)
+
\frac{q_t}{2\mathcal K}
\langle\nabla J(\theta_t),z_t\rangle
-
\frac{\mathcal{L}_J}{2}
\frac{q_t^2}{4\mathcal K^2}.
$$
Taking expectation over the randomness used to construct $z_t$, with the history before iteration $t$ held fixed, and using
$q_t=\mathbb{E}_t[\langle\nabla J(\theta_t),z_t\rangle]$, gives
$$
\mathbb{E}_t\left[
J\left(\theta_t+\frac{q_t}{2\mathcal K}z_t\right)
\right]
\geq
J(\theta_t)+\frac{q_t^2}{2\mathcal K}
-\frac{\mathcal{L}_J}{2}\frac{q_t^2}{4\mathcal K^2}
\geq
J(\theta_t)+\frac{3q_t^2}{8\mathcal K}.
$$
Consequently,
$$
\mathbb{E}_t\left[
\max_{s\in I_t}J(\theta_t+sz_t)
\right]
\geq
J(\theta_t)+\frac{3q_t^2}{8\mathcal K}.
$$
By Lemma~\ref{lemma:surrogate-vs-empiric}, with conditional probability at least $1-\delta$,
$$
J(\theta_{t+1})
=
J(\theta_t+s_tz_t)
\geq
\max_{s\in I_t}J(\theta_t+sz_t)-\Delta_t.
$$
On the remaining event, since both steps belong to $I_t=[0,R_t]$, Assumption~\ref{ass:smooth-bounded} gives
$$
\max_{s\in I_t}J(\theta_t+sz_t)-J(\theta_{t+1})
\leq
R_t\|\nabla J(\theta_t)\|+\frac{\mathcal{L}_J}{2}R_t^2.
$$
By Young's inequality,
$$
R_t\|\nabla J(\theta_t)\|+\frac{\mathcal{L}_J}{2}R_t^2
\leq
\frac{\|\nabla J(\theta_t)\|^2}{2\mathcal K}
+\frac{\mathcal K+\mathcal{L}_J}{2}R_t^2
\leq
\frac{\|\nabla J(\theta_t)\|^2}{2\mathcal K}+\Delta_t,
$$
where the final inequality uses $\mathcal K\geq\mathcal{L}_J$ and $\Delta_t\geq\mathcal K R_t^2$.
Since $\|\nabla J(\theta_t)\|^2$ is fixed given the history before iteration $t$, and the remaining event has conditional probability at most $\delta$, taking $\mathbb{E}_t$ over both cases therefore gives
$$
\mathbb{E}_t[J(\theta_{t+1})]
\geq
J(\theta_t)+\frac{3q_t^2}{8\mathcal K}
-\mathbb{E}_t[\Delta_t]
-\frac{\delta}{2\mathcal K}\|\nabla J(\theta_t)\|^2.
$$
Assumption~\ref{ass:alignment} gives
$q_t^2\geq\mu^2\|\nabla J(\theta_t)\|^2$, and therefore
$$
\frac{\chi}{8\mathcal K}
\|\nabla J(\theta_t)\|^2
\leq
\mathbb{E}_t[J(\theta_{t+1})]-J(\theta_t)
+\mathbb{E}_t[\Delta_t].
$$
Taking full expectation yields
$$
\frac{\chi}{8\mathcal K}
\mathbb{E}\|\nabla J(\theta_t)\|^2
\leq
\mathbb{E}[J(\theta_{t+1})]
-
\mathbb{E}[J(\theta_t)]
+
\widehat{\Delta}_t,
$$
where $\widehat{\Delta}_t=\mathbb{E}[\Delta_t]$.
Summing from $t=0$ to $T-1$ and using Assumption~\ref{ass:smooth-bounded} yields
$$
\frac{\chi}{8\mathcal K}
\sum_{t=0}^{T-1}\mathbb{E}\|\nabla J(\theta_t)\|^2
\leq
J^\star-J(\theta_0)
+
\sum_{t=0}^{T-1}\widehat{\Delta}_t.
$$
Finally,
$$
\min_{0\leq t<T}\mathbb{E}\|\nabla J(\theta_t)\|^2
\leq
\frac{1}{T}\sum_{t=0}^{T-1}
\mathbb{E}\|\nabla J(\theta_t)\|^2,
$$
and therefore, using
$\tilde{\Delta}_T
=
\frac{1}{T}\sum_{t=0}^{T-1}\widehat{\Delta}_t$,
$$
\min_{0\leq t<T}\mathbb{E}\|\nabla J(\theta_t)\|^2
\leq
\frac{8\mathcal K(J^\star-J(\theta_0))}{\chi T}
+
\frac{8\mathcal K
\tilde{\Delta}_T}{\chi}.
$$
\end{proof}



\subsection{Finite Difference Lemmas}

\smartparagraph{Finite-Difference Theorems and Proofs.} We provide the proofs for two finite-difference approximations used by ZFO in the form of Lemmas.

\begin{lemma}\label{lemma:second-order-fin-dif} Let $\phi(s)$ be four times differentiable in an $\epsilon$-neighborhood, then:
    $$\widehat{\phi''}(0)\coloneq\frac{\phi(\epsilon) -2\phi(0)+\phi(-\epsilon)}{\epsilon^2}=\phi''(0)+\mathcal{O}(\epsilon^2)$$
\end{lemma}

\begin{proof}
The Taylor expansion of $\phi(s)$ gives

$$\phi(s)=\phi(0)+s\phi'(0)+s^2\frac{\phi''(0)}{2}+s^3\frac{\phi'''(0)}{6}+s^4\frac{\phi^{(4)}(\xi_s)}{24},$$

hence

$$\phi(\epsilon)=\phi(0)+\epsilon \phi'(0)+\epsilon^2\frac{\phi''(0)}{2}+\epsilon^3\frac{\phi'''(0)}{6}+\epsilon^4\frac{\phi^{(4)}(\xi_\epsilon)}{24}$$

and

$$\phi(-\epsilon)=\phi(0)-\epsilon \phi'(0)+\epsilon^2\frac{\phi''(0)}{2}-\epsilon^3\frac{\phi'''(0)}{6}+\epsilon^4\frac{\phi^{(4)}(\xi_{-\epsilon})}{24}.$$

Adding the two gives

$$\phi(\epsilon)+\phi(-\epsilon)
=2\phi(0)+\epsilon^2 \phi''(0)+\epsilon^4\frac{\phi^{(4)}(\hat{\xi})}{12},$$

subtracting $2\phi(0)$ gives

$$\phi(\epsilon)-2\phi(0)+\phi(-\epsilon)
=\epsilon^2 \phi''(0)+\epsilon^4\frac{\phi^{(4)}(\hat{\xi})}{12},$$

and dividing by $\epsilon^2$ gives

$$\frac{\phi(\epsilon)-2\phi(0)+\phi(-\epsilon)}{\epsilon^2}
=\phi''(0)+\epsilon^2\frac{\phi^{(4)}(\hat{\xi})}{12}
=\phi''(0)+\mathcal{O}(\epsilon^2).$$
\end{proof}

The equality above represents a natural second-order central finite difference; we prove the asymmetric finite difference for Lemma~\ref{lemma:third-order-fin-dif} below,

\begin{lemma}\label{lemma:third-order-fin-dif} Let $\phi(s)$ be five times differentiable in an $\epsilon$-neighborhood, then:
    $$\widehat{\phi'''}(0)\coloneq 3\frac{\phi(\epsilon)-\phi(-\epsilon)-2\epsilon\phi'(0)}{\epsilon^3}=\phi'''(0)+\mathcal{O}(\epsilon^2).$$
\end{lemma}

\begin{proof}
Consider the new expansion of $\phi(s)$

$$\phi(s)
=\phi(0)+s \phi'(0)+s^2\frac{\phi''(0)}{2}+s^3\frac{\phi'''(0)}{6}+s^4\frac{\phi^{(4)}(0)}{24}+s^5\frac{\phi^{(5)}(\xi_s)}{120}.$$

Then taking the difference between $\phi(\epsilon)$ and $\phi(-\epsilon)$ returns

$$\phi(\epsilon)-\phi(-\epsilon)
=2\epsilon \phi'(0)+\epsilon^3\frac{\phi'''(0)}{3}+\epsilon^5\frac{\phi^{(5)}(\hat{\xi})}{60}.$$

Using $\phi'(0)$ gives

$$\phi(\epsilon)-\phi(-\epsilon)-2\epsilon \phi'(0)
=\epsilon^3\frac{\phi'''(0)}{3}+\epsilon^5\frac{\phi^{(5)}(\hat{\xi})}{60},$$

and dividing by $\frac{\epsilon^3}{3}$ returns

$$\frac{3(\phi(\epsilon)-\phi(-\epsilon)-2\epsilon \phi'(0))}{\epsilon^3}
=\phi'''(0)+\epsilon^2\frac{\phi^{(5)}(\hat{\xi})}{20}
=\phi'''(0)+\mathcal{O}(\epsilon^2).$$

\end{proof}
\begin{table*}[t]
  \centering
  \scriptsize
  \setlength{\tabcolsep}{4pt}
  \renewcommand{\arraystretch}{0.95}
  \caption{\small{\textbf{Summary of the models and datasets used in the large-model experiments.}
  Parameter counts denote the total number of model parameters. Dataset sizes are reported as the number of training/evaluation examples used in our experiments. MATH uses the 512/256 train/evaluation subset described in the main text.}}
  \label{tab:model-dataset-summary}
  \scalebox{0.95}{
  \begin{tabular}{@{} l l c l c @{}}
    \toprule
    \textbf{Study} &
    \textbf{Model} &
    \textbf{Model Parameters} &
    \textbf{Dataset} &
    \textbf{Dataset Size (Train / Eval)} \\
    \midrule

    \multirow{12}{*}{Main Study}
      & \multirow{5}{*}{Qwen-2.5-Math-1.5B}
      & \multirow{5}{*}{1.54B}
      & GSM8K      & 7,473 / 1,319 \\
      &&& MATH        & 512 / 256 \\
      &&& SVAMP       & 700 / 300 \\
      &&& AsDiv       & 1,844 / 461 \\
      &&& OpenBookQA  & 4,957 / 500 \\
    \cmidrule(lr){2-5}

      & \multirow{3}{*}{Phi-2}
      & \multirow{3}{*}{2.70B}
      & SVAMP       & 700 / 300 \\
      &&& AsDiv       & 1,844 / 461 \\
      &&& OpenBookQA  & 4,957 / 500 \\
    \cmidrule(lr){2-5}

      & Gemma-2-2B
      & 2.61B
      & SVAMP
      & 700 / 300 \\
    \cmidrule(lr){2-5}

      & \multirow{3}{*}{Llama-3.2-1B}
      & \multirow{3}{*}{1.23B}
      & GSM8K      & 7,473 / 1,319 \\
      &&& AsDiv      & 1,844 / 461 \\
      &&& OpenBookQA & 4,957 / 500 \\

    \midrule

    \multirow{4}{*}{Qwen Sub-Study}
      & \multirow{4}{*}{Qwen-2.5-Math-1.5B}
      & \multirow{4}{*}{1.54B}
      & ARC-Challenge & 1,119 / 1,172 \\
      &&& StrategyQA    & 2,061 / 229 \\
      &&& FOLIO         & 1,001 / 204 \\
      &&& CODAH         & 2,498 / 278 \\

    \bottomrule
  \end{tabular}}
  \vspace{1mm}

\end{table*}

\begin{table*}[b]
\centering
\small

\begin{minipage}[t]{0.36\textwidth}
\centering
\setlength{\tabcolsep}{5pt}
\renewcommand{\arraystretch}{0.98}
\begin{tabular}{@{}ll@{}}
\toprule
\textbf{Parameter} & \textbf{Setting} \\
\midrule
Optimizer & AdamW \\
Policy & \texttt{grpo} \\
Temperature & 0.8 \\
Top-$p$ & 0.95 \\
KL coefficient $\beta^{\star}$ & 0.02 \\
Perturbation scale $\epsilon$ & 1e-3 \\
Seed & \{100, 101, 102\} \\
\bottomrule
\end{tabular}
\captionof{table}{\textbf{Fixed experimental settings} used across all large language RLVR-ZFO runs.}
\label{tab:fixed_config}
\end{minipage}
\hspace{0.035\textwidth}
\begin{minipage}[t]{0.58\textwidth}
\centering
\setlength{\tabcolsep}{4.5pt}
\renewcommand{\arraystretch}{1.58}
\begin{tabular}{@{}cccccc@{}}
\toprule
\textbf{Config} & \textbf{Batch} & \textbf{Group} & \textbf{Epochs} & \textbf{Max Len.} & \textbf{Output Len.} \\
\midrule
\emph{a} & 4 & 4 & 10 & 512 & 192 \\
\emph{b} & 2 & 2 & 8  & 384 & 192 \\
\emph{c} & 8 & 2 & 8  & 384 & 192 \\
\emph{d} & 4 & 2 & 5  & 256 & 192 \\
\bottomrule
\end{tabular}
\captionof{table}{\textbf{Training configuration types used across runs.} Batch represents batch size, Group represents group size, Max Len. and Output Len. represent Max Length and Maximum Output Length.}
\label{tab:config_types}
\end{minipage}

\end{table*}



\begin{table*}
\centering
\small
\setlength{\tabcolsep}{4pt}
\renewcommand{\arraystretch}{0.9}
\caption{\textbf{Hyperparameter Configuration.} Learning rates, bounds, and configuration types used for the main empirical study. FO LR and ZO LR denote the learning rates used for first-order and zeroth-order fine-tuning, respectively. BT2, BT3, BP2, and BP3 represent the selected bounds for Taylor2, Taylor3, Pad\'e2, and Pad\'e3, respectively. Learning rate search is over $\{1\mathrm{e}-6, 5\mathrm{e}{-6}, 1\mathrm{e}{-5}\}$ and bound search is over $\{3,5, 10\}$. ALL indicates that every lr for MeZO produced the same result. }
\label{tab:config_bounds}

\begin{tabular}{@{}llccccccc@{}}
\toprule

\multirow{2}{*}{\textbf{Model}} &
\multirow{2}{*}{\textbf{Dataset}} &
\multicolumn{7}{c}{\textbf{Configuration}} \\

\cmidrule(lr){3-9}

& & FO LR & ZO LR & BT2 & BT3 & BP2 & BP3 & Config Type \\

\midrule

\multirow{5}{*}{\modelname{Qwen-2.5-Math-1.5B}}
 & \datasetname{GSM8K}      & 5e-6 & ALL & 10 & 5 & 10 & 10 & \emph{a} \\
 & \datasetname{MATH}       & 5e-6 & 1e-6 & 5 & 5 & 5 & 5 & \emph{c} \\
 & \datasetname{SVAMP}      & 1e-5 & 5e-7 & 3 & 5 & 3 & 3 & \emph{a} \\
 & \datasetname{AsDiv}      & 5e-6 & 5e-7 & 5 & 5 & 5 & 5 & \emph{c} \\
 & \datasetname{OpenBookQA} & 1e-5 & 1e-7 & 10 & 5 & 5 & 5 & \emph{a} \\ 

\midrule

\multirow{3}{*}{\modelname{Phi-2}}
 & \datasetname{SVAMP}      & 5e-6 & 5e-7 & 3 & 5 & 5 & 3 & \emph{a} \\
 & \datasetname{AsDiv}      & 5e-6 & 1e-7 & 3 & 3 & 5 & 3 & \emph{d} \\
 & \datasetname{OpenBookQA} & 5e-6 & 1e-7 & 3 & 5 & 3 & 3 & \emph{d} \\

\midrule

\multirow{1}{*}{\modelname{Gemma-2-2B}}
 & \datasetname{SVAMP}      & 1e-6 & 5e-9 & 3 & 3 & 3 & 3 & \emph{b} \\

\midrule

\multirow{3}{*}{\modelname{Llama-3.2-1B}}
 & \datasetname{GSM8K}      & 5e-6 & 1e-7 & 3 & 3 & 3 & 3 & \emph{a} \\
 & \datasetname{AsDiv}      & 5e-6 & 5e-8 & 5 & 5 & 3 & 3 & \emph{c} \\
 & \datasetname{OpenBookQA} & 1e-6 & ALL & 5 & 5 & 5 & 3 & \emph{b} \\

\bottomrule
\end{tabular}
\end{table*}

\section{Addendum to Experimental Results} \label{sec:perf_insights}

This addendum collects supplementary diagnostics that complement the main results. We first validate the ZFO step-selection mechanism on two low-dimensional toy RLVR environments, then discuss the local-approximation behavior of each ZFO method on a panel of analytic test functions, and finally report cross-setting win ratios that summarize how consistently each variant improves over the FO baseline.

\subsection{Configuration}
\label{sec:configuration}
This section specifies the objective and implementation choices used throughout our experiments. First, we define the GRPO surrogate, optimized by each method, then report the hyperparameter configurations used for the main large-model runs, Qwen sub-study, and toy validation.
 
\subsubsection{Surrogate Objective Formulation}
\label{subsec:objective_formulation}
\smartparagraph{GRPO Surrogate.} For this paper, the surrogate $F_t$ instantiates the GRPO objective. Let $\mathcal{P}_t$ denote the set of prompts in the current batch, with $|\mathcal{P}_t|$ the batch size and $G$ the number of examples sampled per prompt. For each prompt $i\in\{1,\ldots,|\mathcal{P}_t|\}$ and each sampled example $j\in\{1,\ldots,G\}$, let $\tau_{i,j,t}$ denote the sampled example, $|\tau_{i,j,t}|$ its length, and $(x_{i,j,v},y_{i,j,v})$ the state-action pair at token position $v$. The surrogate is
$$
F_t(\theta;S_t)
=
\frac{1}{|\mathcal{P}_t|G}
\sum_{i=1}^{|\mathcal{P}_t|}
\sum_{j=1}^{G}
\widetilde{A}_{i,j,t}\log\pi_\theta(\tau_{i,j,t})
-
\kappa\cdot\mathrm{KL}(\pi_\theta\|\pi_{\rm ref}),
$$
where $\kappa>0$ is the KL coefficient, $\widetilde{A}_{i,j,t}$ is the group-relative advantage given as
$$
\widetilde{A}_{i,j,t}
=
\rho(\tau_{i,j,t})
-
\frac{1}{G}\sum_{k=1}^{G}\rho(\tau_{i,k,t}),
$$
and $\rho$ is the verifiable reward given by
$$
\rho(\tau_{i,j,t})
=
\mathds{1}_{\tau_{i,j,t}\in\mathcal{C}_t},
$$
with $\mathcal{C}_t$ denoting the set of sampled examples deemed correct by the verifier at iteration $t$. The KL penalty is computed at the token level and averaged across the batch:
$$
\mathrm{KL}(\pi_\theta\|\pi_{\rm ref})
=
\frac{1}{|\mathcal{P}_t|G}
\sum_{i=1}^{|\mathcal{P}_t|}
\sum_{j=1}^{G}
\frac{1}{|\tau_{i,j,t}|}
\sum_{v=1}^{|\tau_{i,j,t}|}
\bigl[
\log\pi_\theta(y_{i,j,v}\mid x_{i,j,v})
-
\log\pi_{\rm ref}(y_{i,j,v}\mid x_{i,j,v})
\bigr].
$$
\label{subsec:surrogate_config}

\begin{table*}
\centering
\small
\setlength{\tabcolsep}{4pt}
\renewcommand{\arraystretch}{0.9}
\caption{\textbf{Learning rate, bound, and configuration type used for the Qwen sub-study on additional reasoning benchmarks.} Learning rate search is over $\{$5e-6, 1e-5$\}$ and bound search is over $\{3,5\}$.}

\begin{tabular}{@{}llcccccc@{}}
\toprule

\multirow{2}{*}{\textbf{Model}} &
\multirow{2}{*}{\textbf{Dataset}} &
\multicolumn{6}{c}{\textbf{Configuration}} \\

\cmidrule(lr){3-8}

& & Learning Rate & BT2 & BT3 & BP2 & BP3 & Config Type \\

\midrule

\multirow{4}{*}{\modelname{Qwen-2.5-Math-1.5B}}
 & \datasetname{ARC-Challenge} & 1e-5 & 3 & 5 & 3 & 3 & \emph{c} \\
 & \datasetname{StrategyQA}    & 1e-5 & 5 & 5 & 3 & 5 & \emph{c} \\
 & \datasetname{FOLIO}         & 1e-5 & 3 & 3 & 3 & 3 & \emph{c} \\
 & \datasetname{CODAH}         & 1e-5 & 5 & 5 & 5 & 5 & \emph{c} \\

\bottomrule
\label{tab:qwen_config_bounds}
\end{tabular}
\end{table*}

\smartparagraph{Implementation-Level Interpretation of $\alpha$ and $\beta$.}
The implementation uses an equivalent parameterization of the search interval
defined in the methodology. Let $\theta_t$ and $\widehat{\theta}_t$ denote the
parameters before and after a tentative AdamW update with learning rate
$\alpha$, and define the realized optimizer displacement $d_t=\widehat{\theta}_t-\theta_t.
$

For $d_t\neq 0$, searching over $\theta_t+a d_t$ with $a\in[0,\beta]$
is equivalent to searching over $\theta_t+s z_t$, where
$$
z_t=\frac{d_t}{\|d_t\|},
\quad
s=a\|d_t\|,
\quad
s\in[0,R_t],
\quad
R_t=\beta\|d_t\|.
$$
These are two coordinate representations of the same parameter-space
interval. For a fixed parameter vector, gradient, and optimizer state, the
exact-arithmetic AdamW displacement scales linearly with $\alpha$.
Under these conditions, matched-product pairs such as
$(\alpha,\beta)=(10^{-6},5)$ and $(5\times10^{-6},1)$ define the same
physical search interval.

This proportionality need not hold in the numerical implementation.
The tentative update is applied to finite-precision parameters, and
$d_t$ is obtained by subtracting the stored parameter vectors.
Consequently, rounding can cause a displacement to differ. In particular, multiplying a small, already-rounded
update by $\beta$ need not reproduce an update formed using a larger base
learning rate. Matched $\alpha\beta$ products therefore need not yield
identical realized search radii or directions. Numerical safeguards can
introduce additional departures from scaling equivalence, and subsequent
differences in parameters, sampled responses, and optimizer states can
amplify these discrepancies over training. The observed matched-product
differences thus concern the implemented numerical procedure and do not
contradict the exact-arithmetic equivalence.

\smartparagraph{LLM implementation.} GRPO supplies the optimizer direction and first-order coefficient. Symmetric probes evaluate a separate reward-minus-KL score on fixed sampled sequences, with rewards held constant. The probe values and GRPO slope jointly define the local model used for step selection. The theory analyzes an idealized ZFO setting in which the directional derivative and function probes are evaluations of the same sampled surrogate. The RLVR implementation approximates this setting by holding sampled responses and rewards fixed during the probes.

\smartparagraph{Empirical Candidate Selection.} Interior search controls which candidate steps are considered within the
same search interval. Using the displacement parameterization above,
Taylor2 and Taylor3 compare the endpoints with feasible real stationary
points, whereas the accepted Pad\'e2 $[1/1]$ and Pad\'e3 $[1/2]$ models
compare the two endpoints. A regular $[1/1]$ model is monotone or constant,
so its maximum can be attained at an endpoint. A regular $[1/2]$ model
can instead possess interior extrema; endpoint selection for Pad\'e3
is therefore an empirical restriction on its candidate set, although the safeguarded Pad\'e3 procedure can still select an interior
step through its Taylor2 fallback when a pole or near-pole is detected.

\subsubsection{Empirical Configurations}\label{subsec:emp_config}
All large-model runs share the fixed hyperparameters listed in Table~\ref{tab:fixed_config}. We hold these values constant so that observed differences in performance can be attributed to the update method as opposed to incidental variance in sampling, regularization, or optimization settings. Each reported result is averaged over the three seeds $\{ 100, 101, 102 \},$ and the standard deviations reported in Tables~\ref{tab:empirics} and ~\ref{tab:qwen_substudy} are based on these various runs.

We evaluate the models after each epoch, and all runs use the same early-stopping rule; we terminate the training when evaluation accuracy is unchanged for three consecutive epochs. Otherwise, we continue training until the maximum epoch count specified for the corresponding configuration. Per-run hyperparameters differ across model-dataset pairs and include the learning rate $\alpha$, the bound hyperparameter $\beta,$ and the selected configuration from Table~\ref{tab:config_types}. We summarize four configuration types in Table~\ref{tab:config_types}, which represent a performance-versus-cost trade-off. Config \emph{a} is the default, with the largest group size and longest context length. Configs \emph{b}–\emph{d} are progressively cheaper alternatives selected when \emph{a} is infeasible based on either memory or time constraints. Config \emph{b} reduces both batch and group sizes for memory-constrained runs, \emph{c} trades group size for a larger batch when throughput is more valuable than rollout diversity, and \emph{d} additionally shortens the context and reduces epochs.

These per-run details are reported in Table~\ref{tab:config_bounds} for the main results, including the FO learning rate, ZO learning rate, and each ZFO bound $\beta$. The FO learning rate is selected from the sweep $\{$1e-6, 5e-6, 1e-5
$\}$, and the corresponding bound for each ZFO method is then chosen from a validated sweep over $\{ 3, 5, 10 \}$. The ZO learning rate is tuned independently from the sweep \{$1e-8, 5e-8, 1e-7, 5e-7
\}$, since ZO methods generally require a smaller effective step scale than FO updates. The performance reported is the mean verifiable reward on the evaluation set. The Qwen sub-study from Table~\ref{tab:qwen_substudy} follows a similar protocol but fixes config \emph{c} across runs (a memory-conservative choice given the sub-study sweeps over four datasets), restricts the learning rate sweep to $\{$5e-6, 1e-5
$\}$, and limits the bound sweep to $\{3, 5\}.$ The Qwen sub-study configurations are reported in Table~\ref{tab:qwen_config_bounds}.

The toy models in Figure~\ref{fig:toy_collection} use a lightweight configuration designed to validate the ZFO step-selection mechanism in a controlled low-dimensional setting. Both environments share a context dimension of 6, a sequence length of 4, and a two-layer policy network of hidden width 48. Each run uses a dataset of 200 contexts trained for 50 steps with batch size $16$, group size $G=4$, perturbation scale $\epsilon=10^{-2},$ and KL coefficient of 0.1. The bound hyperparameter $\beta$ is fixed to 1, so the ZFO search interval is restricted to steps no larger than the standard FO update.

Table~\ref{tab:model-dataset-summary} depicts the models and datasets used throughout our large-model experiments, along with their corresponding parameter counts and dataset sizes. We consider models ranging from 1.23B to 2.7B parameters across a diverse set of mathematical reasoning, commonsense reasoning, and knowledge-based benchmarks. We additionally report the datasets used in our Qwen sub-study to provide a complete summary of the main experimental settings considered.

\begin{table}
\centering
\footnotesize
\caption{\textbf{Compact summary of the Classical Optimization Problems used in Figure~\ref{fig:optimization_beyond_llms}.} Before optimization, each listed initial position $\theta_0$ is perturbed independently per seed as $\theta_{0,s}=\theta_0+\xi_s$, where $\xi_s\sim\mathcal{N}(0,0.04I)$ applies independent Gaussian noise (standard deviation $0.2$) to every coordinate, drawn from a dedicated, task-specific random stream and shared across all FO, ZFO, and GeN methods evaluated on that task--seed pair.}
\label{tab:optimization_beyond_llms}
\begin{tabular}{ccc}
\toprule
Problem & Objective & Initial position \\
\midrule
Linear least squares &
$\min_\theta\dfrac{1}{2n}\lVert A\theta-y\rVert^2$ &
$\theta_0=\mathbf{0}_{20}$ \\[6pt]
Nonlinear least squares &
$\min_{a,b,c}\dfrac{1}{2N}\sum_i\big(a\,e^{bx_i}+c-y_i\big)^2$ &
$(a_0,b_0,c_0)=(0.4,\,0.8,\,1.0)$ \\[6pt]
Logistic regression &
$\min_\theta\dfrac{1}{n}\sum_i\big[\log(1+e^{x_i^\top\theta})-y_ix_i^\top\theta\big]+\dfrac{\lambda}{2}\lVert\theta\rVert^2$ &
$\theta_0=\mathbf{0}_{16}$ \\[6pt]
Low-rank matrix factorization &
$\min_{U,V}\dfrac{1}{2mn}\lVert UV^\top-M\rVert_F^2+\dfrac{\lambda}{2|\theta|}\lVert\theta\rVert^2$ &
$(U_0,V_0)\sim\mathcal{N}(0,0.15^2)$ \\[6pt]
Rosenbrock &
$\min_{x_1,x_2}100(x_2-x_1^2)^2+(1-x_1)^2$ &
$(x_1,x_2)_0=(-1.2,\,1.0)$ \\[6pt]
Beale &
$\min_{x_1,x_2}\sum_{k=1}^{3}\big(c_k-x_1+x_1x_2^k\big)^2$ &
$(x_1,x_2)_0=(-1.0,\,1.0)$ \\
\bottomrule
\end{tabular}
\end{table}

Our MeZO baseline uses a forward-only, two-sided finite-difference update along an unnormalized Gaussian perturbation direction. At each iteration, sampled responses, group-centered reward advantages, and reference-policy log probabilities are held fixed across the positive and negative perturbations. The evaluated objective combines advantage-weighted completion log probabilities with a nonnegative \(k_3\) KL penalty, averaged over completion tokens within each sequence. Both terms use completion-only masking, and responses are decoded from the end of the padded prompt. Objective evaluations run in evaluation mode, with log-ratio clipping for numerical stability. The perturbation direction is regenerated from the same random seed for both probes and the parameter update, avoiding backpropagation.

\smartparagraph{Interpreting MeZO Performance.}
Beyond low average performance, MeZO reaches zero final accuracy in
some settings. In two settings, all three seeds decline from nonzero
first-epoch accuracy to zero, indicating deterioration during training.
Random zeroth-order updates, weak group-relative rewards, and learning rate sensitivity may contribute to this instability. 
These explanations remain hypotheses rather than isolated causes,
and the results concern MeZO's specific adaptation to RLVR.

\subsection{Additional Insights}
\label{sec:add_insights}

\smartparagraph{Optimization Beyond Neural Network Training.} To evaluate whether ZFO extends beyond neural network optimization, we consider six classical optimization problems: linear least squares, nonlinear least squares, logistic regression, low-rank matrix factorization, Rosenbrock, and Beale. A summary of the problems can be found in Table~\ref{tab:optimization_beyond_llms}. We evaluate Taylor2, Taylor3, Padé2, and Padé3 with both SGD and AdamW-style directions, alongside their first-order baselines and GeN~\cite{bu2025gradient} and present the results in Figure~\ref{fig:optimization_beyond_llms}. Each method is independently tuned over the same 18 learning rates, \(\{1,2,5\}\times10^{-5,\ldots,0}\), while ZFO methods use a fixed bound \(\beta=3\) and perturbation radius \(\epsilon=10^{-3}\). GeN uses \(\gamma=0.9\) and \(\Phi=1\) for fair comparison against the ZFO methods. Each experiment runs for 100 gradient updates over three seeds.
\begin{figure}[h]
    \centering
    \includegraphics[width=0.98\linewidth]{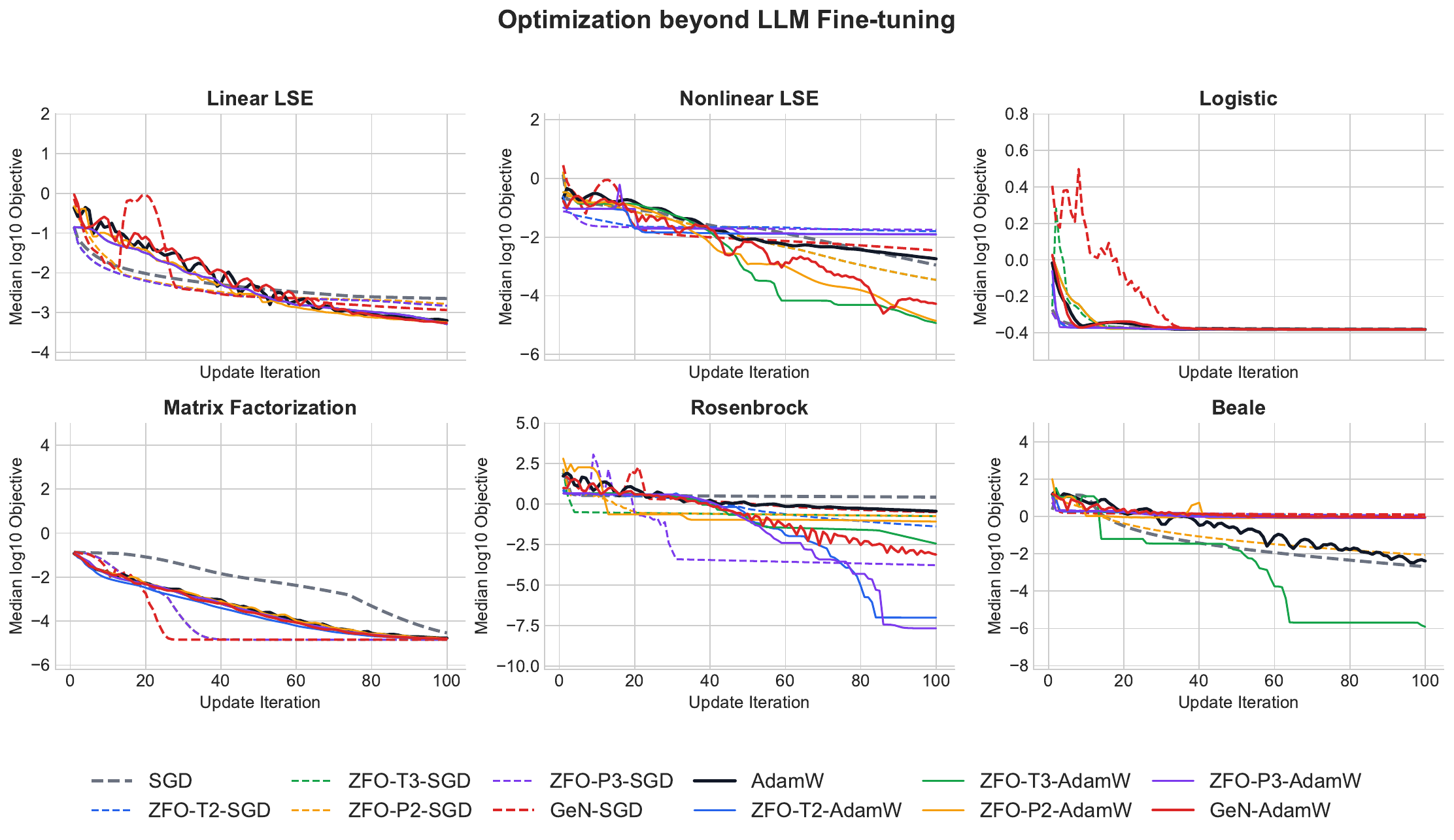}
    \caption{\textbf{Median $\log_{10}$ objective over three seeds for six optimization problems.} Solid curves use AdamW directions, and dashed curves use SGD directions. Each method uses its best learning rate from the same 18-point search. Lower values indicate better optimization.}
    \label{fig:optimization_beyond_llms}
\end{figure}

The results show that ZFO’s benefits are substantial but depend on both the local model and direction. Taylor$_3$–AdamW attains the lowest final median objectives on nonlinear least squares and Beale, whereas Padé$_3$–AdamW performs best on Rosenbrock. Several ZFO and GeN variants also accelerate matrix factorization relative to ordinary SGD. On linear least squares, most adaptive methods improve moderately over their corresponding first-order baselines, whereas logistic regression converges to essentially the same objective across methods. Some combinations also stall or underperform, particularly on the nonconvex objectives. Thus, the experiment supports ZFO’s applicability beyond LLMs and its ability to yield order-of-magnitude improvements in favorable geometries.

\smartparagraph{All Roads Lead to $\phi$.} Figure~\ref{fig:local_functions} examines the four ZFO surrogates, \texttt{Taylor$_2$}, \texttt{Taylor$_3$}, \texttt{Padé$_2$}, and \texttt{Padé$_3$}, on a panel of locally analytic functions. In each panel, $f(x_0)$ and $f'(x_0)$ are computed exactly, while the second and third derivatives are approximated by the same finite differences discussed in this paper. The top row consists of functions that are either sub-linear or polynomial on a tight interval; hence the approximations are nearly indistinguishable from the actual function. The bottom row deliberately widens the window by introducing functions that do not have sub-linear properties. Across both rows, the approximants closely agree with the true objective near the expansion point, while their behavior diverges as the approximation interval widens and the objective becomes more nonlinear.


\begin{figure}[h]
    \centering
    \includegraphics[width=0.99\linewidth]{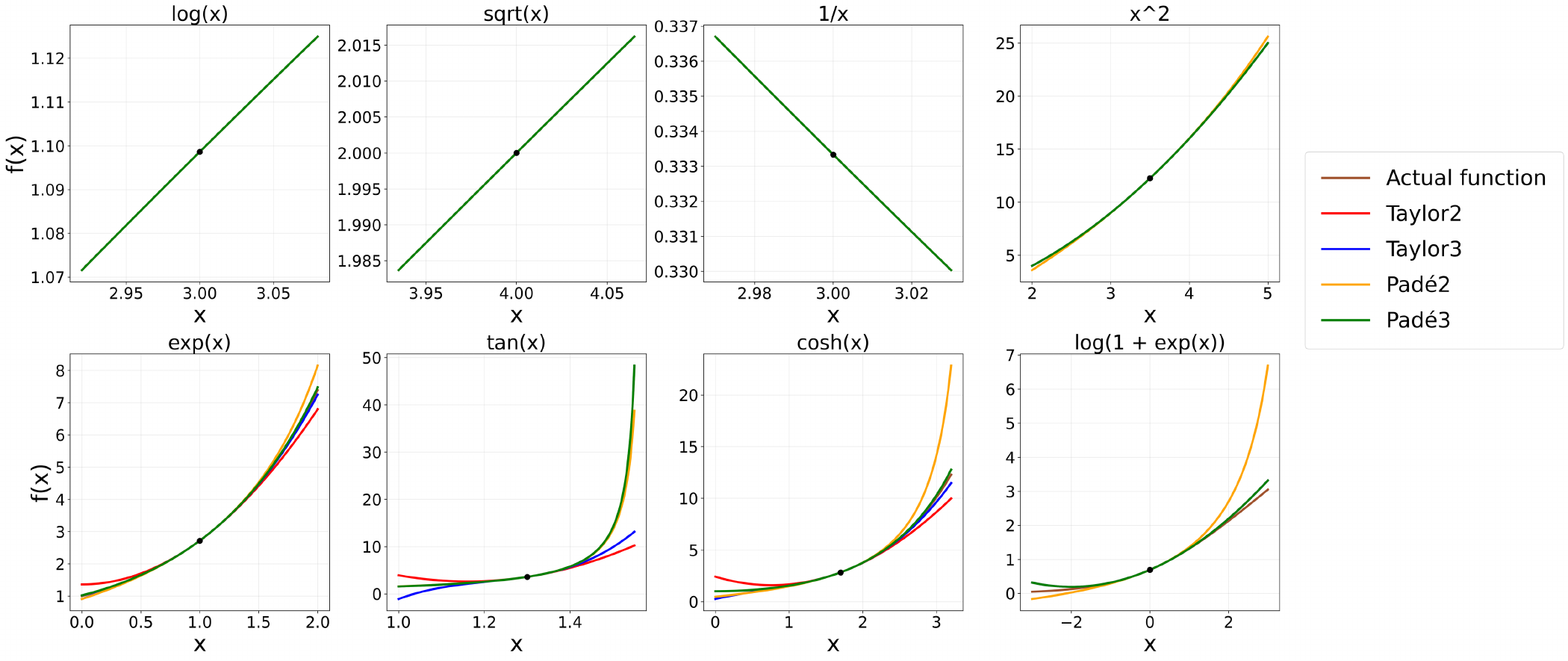}
    \caption{\textbf{ZFO approximations on a variety of locally analytic functions.} Each panel plots \texttt{Taylor$_2$}, \texttt{Taylor$_3$}, \texttt{Padé$_2$}, and \texttt{Padé$_3$} around an $x_0.$ The top row consists of panels where all four interpolants are visually indistinguishable from the actual function. The bottom row consists of panels that expose the qualitative differences between the Taylor and Padé approximants.}
    \label{fig:local_functions}
\end{figure}

\smartparagraph{Toy Model Validation.} We first validate the ZFO framework on two low-dimensional toy RLVR environments shown in Figure~\ref{fig:toy_collection}. The first toy model, \emph{Grouped Sequence Bandit}, is a sequence-decision task that requires a learning algorithm to choose a short sequence of tokens to match a hidden target sequence. Specifically, the model samples a complete sequence based on a small input and updates the parameters based on the reward feedback. The second model, \emph{Grouped Path-Matching}, uses a similar idea, but makes the reward structure more dependent on the entire sequence path. Instead of rewarding independent tokens, the policy receives a higher reward when its sampled actions remain aligned with the hidden target path. Both environments use a stochastic policy trained with grouped advantages and KL regularization against a frozen reference, replicating the surrogate structure of $F_t(\theta; \mathcal{S}_t) $ used in our larger experiments. Each method is averaged over 5 seeds with the bound hyperparameter $\beta=1$; the remainder of the configuration can be found in \S\ref{subsec:emp_config}.

Figure~\ref{fig:toy_collection} reports the training curves and final performance across both environments. Across both toy environments, the ZFO methods improve over the FO and ZO baselines, supporting the central claim that additional function evaluations can improve step-size selection along a fixed direction. Importantly, these improvements occur with the same search scale, since all methods use $\beta=1.$  Thus, the ZFO methods are not benefiting from a larger trust region; rather, they are using local curvature behavior to select more effective intermediate steps within the same update interval. This is most visible in the separation between \texttt{Taylor$_2$}, \texttt{Taylor$_3$}, and \texttt{Padé$_3$} relative to FO, showing that the additional curvature information can meaningfully alter the chosen step even when the learning rate is fixed. The behavior of \texttt{Padé2} further supports this interpretation; among the ZFO methods, \texttt{ Padé2} is the most restricted in the sense that it contains only a linear numerator and denominator. Consequently, when $\beta=1$, \texttt{Padé2} often selects a step close to the standard FO update, making it the ZFO variant most similar to FO. This trend is reflected in the final performance summaries with \texttt{Taylor$_3$} performing best in both environments followed closely by \texttt{Taylor$_2$} and \texttt{Padé$_3$}. Overall, these toy results provide validation of the ZFO step-selection mechanism before moving to higher-dimensional models.

\begin{figure}[t]
    \centering
    \includegraphics[width=0.98\linewidth]{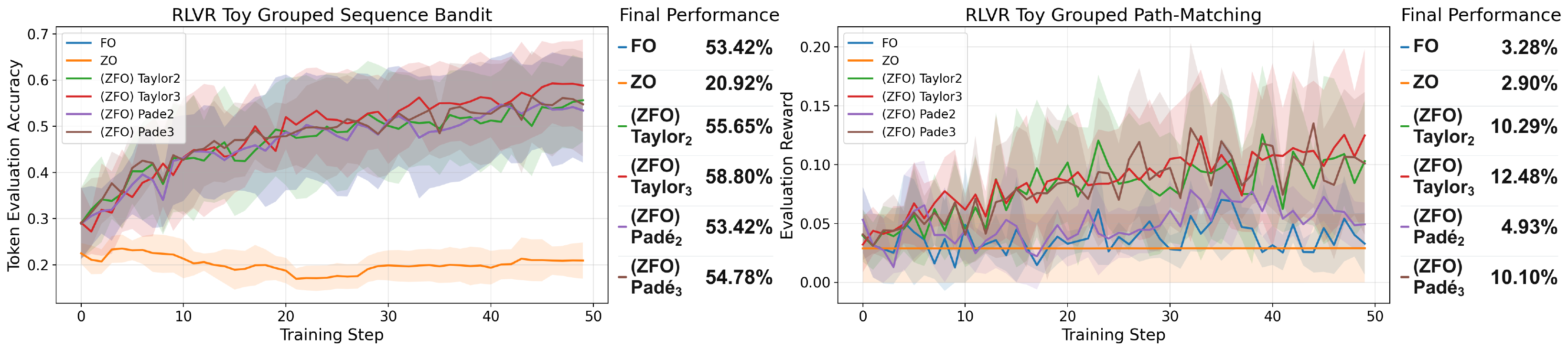}
    \caption{\textbf{Toy RLVR validation on the \emph{Grouped Sequence Bandit} and \emph{Grouped Path-Matching} environments.} ZFO methods improve over FO and ZO baselines under the same search scale $\beta=1$, showing that local curvature models can select more effective steps along update directions.}
    \label{fig:toy_collection}
\end{figure}

\smartparagraph{Learning Rate Sensitivity Curves.} Figure~\ref{fig:empirical_curves} establishes that ZFO's gains arise from intelligent step selection rather than a hidden learning-rate effect. In each panel, we plot ZFO trajectories alongside FO FT runs surrounding the base learning rate to compare against ZFO runs. The three panels highlight complementary behaviors; e.g., Qwen on OpenBookQA shows ZFO trajectories outperforming FO runs with the same nominal search regions, reinforcing the extremely large margin in Table~\ref{tab:empirics}. Phi-2 on SVAMP shows ZFO trajectories steadily improving over the FO baselines. Llama on AsDiv shows that the most conservative FO baseline eventually destabilizes during training, whereas ZFO methods remain stable. Together, these curves indicate that ZFO occupies a regime that FO cannot reach by fixed learning rates.

\begin{figure}
    \centering
    \includegraphics[width=0.9\linewidth]{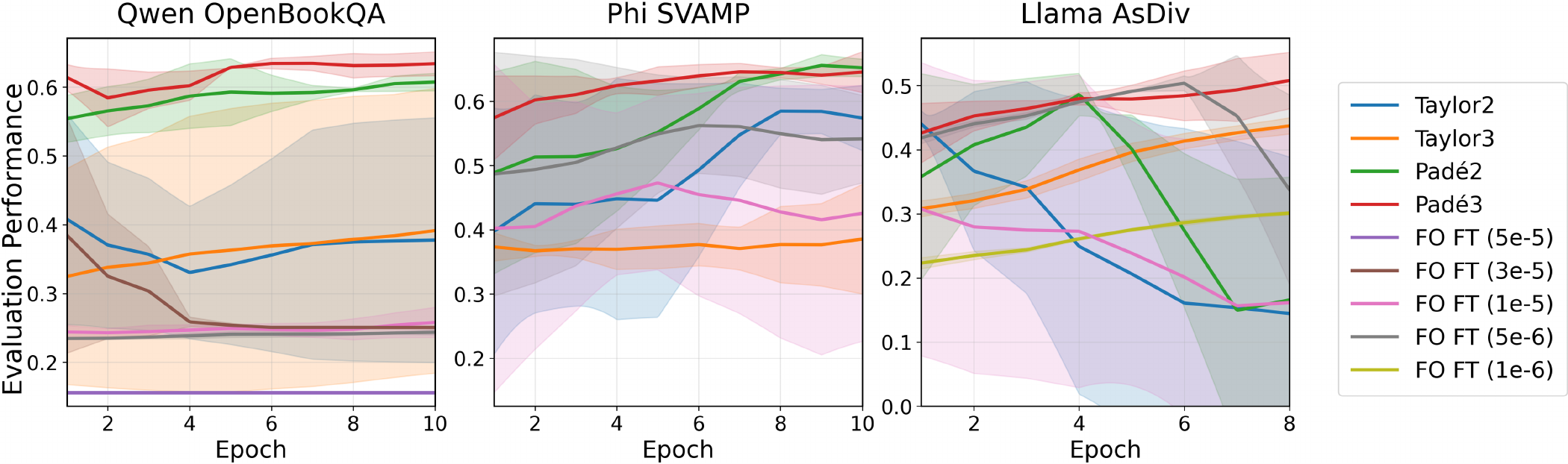}
    \vspace{-1mm}
    \caption{\small{\textbf{Stability of ZFO.} Empirical evaluation curves demonstrating ZFO's capacity to stabilize and improve over fine-tuning with AdamW (FO FT). Demonstrations involve AdamW FT (FO FT) curves of multiple learning rates.}}
    \label{fig:empirical_curves}
    \vspace{-4mm}
\end{figure}

\smartparagraph{First-Order Variability and Training Stability.}
High variability across seeds can reflect late training collapse.
For ASDiv on Llama-3.2-1B, two FO seeds initially exceed $50\%$ accuracy
before falling below $0.4\%$, while the third finishes at $49.07\%$.
The resulting $16.60\pm28.12$ therefore reflects sharply different
training outcomes. The corresponding
Pad\'e3 runs finish between $47.33\%$ and $56.83\%$.
Since ZFO can suppress or shorten a proposed update, its role extends
beyond searching for larger steps: adaptive step selection may also
reduce destructive updates. These observations support a stabilizing
role in some configurations.

The additional FO results further illustrate this sensitivity:
for OpenBookQA on Qwen, linear warmup followed by cosine decay improves
AdamW from $26.07\pm2.53$ to $59.60\pm0.87$,
although the schedule does not improve every dataset. This highlights
the importance of how step sizes evolve during fine-tuning.

\smartparagraph{Bound vs. Learning Rate Ablation.} Figure~\ref{fig:ablation} reports the full ablation over the learning rate $\alpha$ and search bound $\beta$ that was summarized in the main paper. We sweep both hyperparameters across all four ZFO methods on Qwen-2.5-Math-1.5B, evaluating after 10 epochs on subsets (512/256 train/eval) of GSM8K and SVAMP. As discussed in the main paper, at smaller learning rates, increasing $\beta$ produces near-monotonic improvements, while at larger learning rates the gains saturate as the search region $R$ grows too large. Ideally, at a fixed parameter vector, gradient, and optimizer state, matched-product configurations such as $(\alpha = 10^{-6}, \beta = 5)$ and $(\alpha = 5 \times 10^{-6}, \beta = 1)$ define the same search interval in exact arithmetic. Empirically, finite-precision updates and numerical safeguards can break this equivalence, with discrepancies potentially amplified over training; see Section~\ref{subsec:surrogate_config}.

Examining Figure~\ref{fig:ablation} closely, we observe the high-reward region forms a diagonal band that shifts toward smaller $\beta$ as $\alpha$ grows. At $\alpha=10^{-6}$ and $\alpha=5\times10^{-6}$, performance improves steadily as $\beta$ increases from $1$ to $100$ across all four methods. At $\alpha=10^{-5}$, the best results pull back toward small $\beta.$ Combining the largest bound with $\alpha=10^{-5}$ collapses the reward to near-zero for almost all methods.

\begin{figure}
    \centering
    \includegraphics[width=0.99\linewidth]{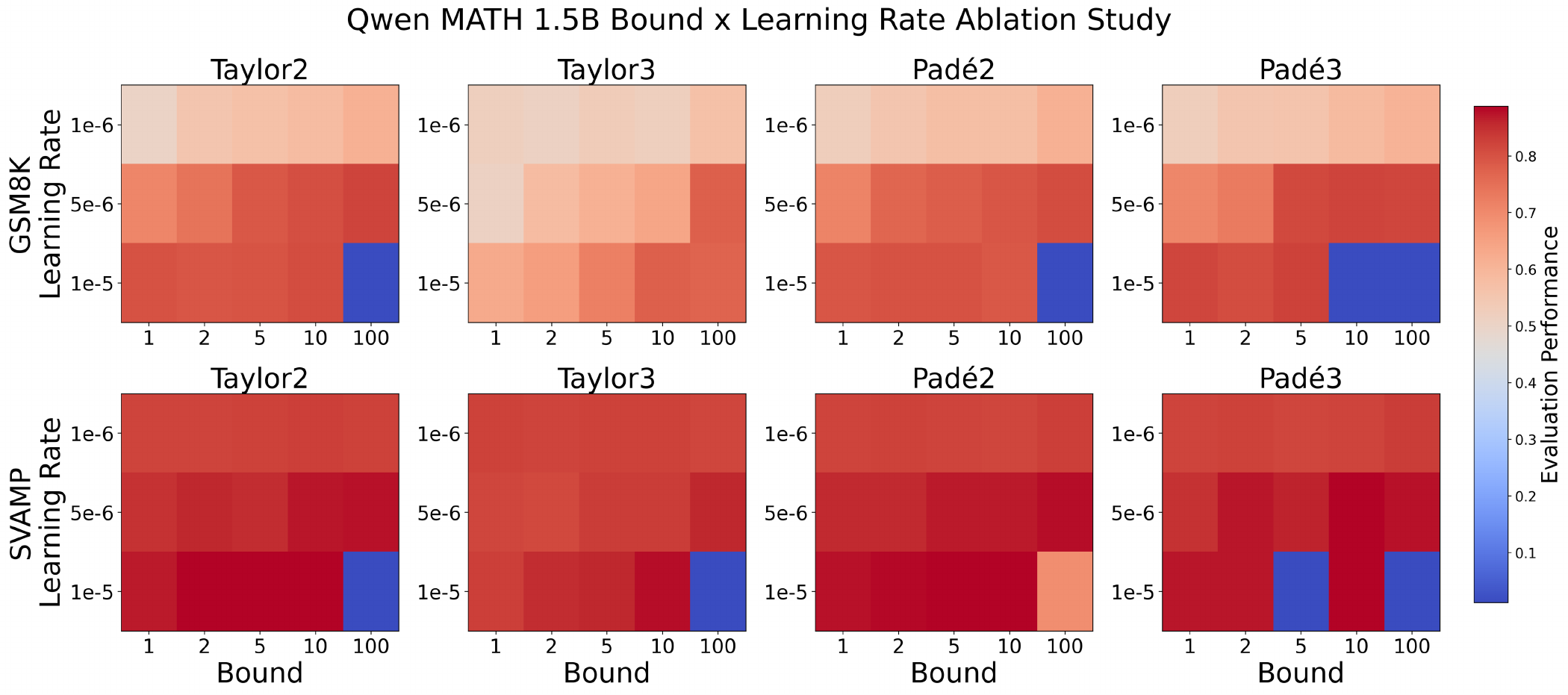}
    \caption{\textbf{Ablation study inspecting the relationship between the learning rate $\alpha$ and search hyperparameter $\beta$ on the reward rate.} We report the performance over a training set of 512 elements and an evaluation set of 256 elements on the Qwen-2.5-MATH-1.5B model and datasets, GSM8K and SVAMP, on the top and bottom row, respectively. From left to right, the panels represent all four ZFO methods (\texttt{Taylor$_2$}, \texttt{Taylor$_3$}, \texttt{Padé$_2$}, \texttt{Padé$_3$}) presented in this paper.}
    \label{fig:ablation}
\end{figure}

\smartparagraph{Cross-Setting Improvement Frequency.} As an additional appendix diagnostic, we summarize how often each ZFO variant improves over the corresponding FO baseline across the reported evaluation settings in Table~\ref{tab:zfo_win_ratios}. Across the 12 main result settings and the 4 Qwen sub-study settings, \texttt{Taylor$_2$}, \texttt{Taylor$_3$}, \texttt{Padé$_2$}, and \texttt{Padé$_3$} improve over FO in 14/16, 7/16, 13/16, and 14/16 settings, respectively. This supports the view that ZFO is driven by a recurring advantage in adaptive step selection across settings.

\begin{table}[t] 
\centering
\small
\setlength{\tabcolsep}{6pt}
\renewcommand{\arraystretch}{1.05}
\begin{tabular}{@{}lccc@{}}
\toprule
\textbf{Method} & \textbf{Main Results Wins} & \textbf{Qwen Sub-Study Wins} & \textbf{Win Ratio} \\
\midrule
\texttt{Taylor$_2$} & 10/12 & 4/4 & 14/16 \\
\texttt{Taylor$_3$} & 6/12  & 1/4 & 7/16 \\
\texttt{Padé$_2$} & 11/12 & 2/4 & 13/16 \\
\texttt{Pad\'e$_3$} & 11/12 & 3/4 & 14/16 \\
\bottomrule
\end{tabular}
\vspace{1mm}
\caption{\textbf{Win ratios for each ZFO method across the main results and Qwen sub-study.} A win denotes a setting in which the method achieves a higher mean accuracy than the corresponding FO baseline.}
\label{tab:zfo_win_ratios}
\end{table}

\paragraph{Additional first-order baselines.}
To compare ZFO with different learning-rate adaptation strategies, we additionally evaluate AdamW with a warmup-cosine schedule~\cite{sgdr} and the parameter-free Prodigy optimizer~\cite{prodigy} in Table~\ref{tab:additional_fo_zfo_extremes}. We preserve the Qwen-2.5-Math-1.5B model, dataset-specific training configurations, GRPO objective, evaluation protocol, and seeds used in the corresponding main and sub-study experiments. For AdamW with cosine decay, the learning rate selected for the constant-rate FO baseline is used as the peak rate; it gets reached by linear warmup over the first \(10\%\) of optimizer steps and then decayed to zero according to a cosine schedule. Prodigy instead uses its recommended configuration with adjustment factor \(1.0\), \(d_{\mathrm{coef}}=1.0\), decoupled weight decay \(0.01\), and no external learning-rate schedule. To contextualize these baselines within the range of behavior exhibited by ZFO, the table also reports the lowest- and highest-mean ZFO variants for each dataset.

\providecommand{\numcell}[1]{{\fontsize{6.6}{7.2}\selectfont #1}}
\providecommand{\modelname}[1]{{\fontsize{6.8}{7.4}\selectfont \textbf{#1}}}
\providecommand{\datasetname}[1]{{\fontsize{6.8}{7.4}\selectfont #1}}

\begin{table*}[t]
\centering
\small
\setlength{\tabcolsep}{5pt}
\renewcommand{\arraystretch}{0.92}
\caption{{\textbf{Comparison with additional first-order baselines.}
AdamW with linear warmup and cosine decay and Prodigy are compared with the
worst- and best-performing ZFO variants for each Qwen-2.5-Math-1.5B
dataset. The ZFO extrema are selected by mean final
evaluation performance across Taylor$_2$, Taylor$_3$, Pad\'e$_2$, and
Pad\'e$_3$; the selected variant is shown in parentheses. The best result in each row is bold.}}
\vspace{-1mm}

\scalebox{0.95}{%
\begin{tabular}{@{}llcccc@{}}
\toprule
\textbf{Model} &
\textbf{Dataset} &
\textbf{AdamW + Cosine} &
\textbf{Prodigy} &
\textbf{Worst ZFO} &
\textbf{Best ZFO} \\
\midrule

\multirow{8}{*}{\modelname{Qwen-2.5-Math-1.5B}}
& \datasetname{GSM8K}
& \numcell{72.88 $\pm$ 0.54}
& \numcell{64.32 $\pm$ 0.90}
& \numcell{75.56 $\pm$ 0.24 (Taylor$_3$)}
& \numcell{\textbf{81.00 $\pm$ 0.11} (Taylor$_2$)} \\

& \datasetname{SVAMP}
& \numcell{86.00 $\pm$ 0.33}
& \numcell{83.22 $\pm$ 0.19}
& \numcell{89.11 $\pm$ 0.51 (Taylor$_3$)}
& \numcell{\textbf{90.89 $\pm$ 1.02} (Pad\'e$_2$)} \\

& \datasetname{AsDiv}
& \numcell{81.98 $\pm$ 0.52}
& \numcell{80.82 $\pm$ 0.30}
& \numcell{88.27 $\pm$ 0.35 (Taylor$_3$)}
& \numcell{\textbf{94.09 $\pm$ 0.29} (Taylor$_2$)} \\

& \datasetname{OpenBookQA}
& \numcell{59.60 $\pm$ 0.87}
& \numcell{38.73 $\pm$ 0.31}
& \numcell{37.80 $\pm$ 17.85 (Taylor$_2$)}
& \numcell{\textbf{64.07 $\pm$ 1.22} (Pad\'e$_3$)} \\

& \datasetname{ARC-Challenge}
& \numcell{55.15 $\pm$ 1.11}
& \numcell{30.80 $\pm$ 0.15}
& \numcell{53.58 $\pm$ 2.77 (Pad\'e$_2$)}
& \numcell{\textbf{58.10 $\pm$ 1.33} (Taylor$_2$)} \\

& \datasetname{StrategyQA}
& \numcell{66.28 $\pm$ 0.51}
& \numcell{60.80 $\pm$ 0.95}
& \numcell{61.67 $\pm$ 1.46 (Taylor$_3$)}
& \numcell{\textbf{70.06 $\pm$ 0.61} (Pad\'e$_3$)} \\

& \datasetname{FOLIO}
& \numcell{34.48 $\pm$ 1.13}
& \numcell{34.31 $\pm$ 0.49}
& \numcell{34.48 $\pm$ 1.13 (Taylor$_3$)}
& \numcell{\textbf{36.60 $\pm$ 3.26} (Taylor$_2$)} \\

& \datasetname{CODAH}
& \numcell{48.32 $\pm$ 14.52}
& \numcell{30.58 $\pm$ 1.57}
& \numcell{41.85 $\pm$ 20.58 (Pad\'e$_3$)}
& \numcell{\textbf{62.83 $\pm$ 3.14} (Taylor$_2$)} \\

\bottomrule
\end{tabular}%
}
\label{tab:additional_fo_zfo_extremes}
\vspace{-5mm}
\end{table*}

Table~\ref{tab:additional_fo_zfo_extremes} shows that the strongest ZFO variant outperforms both AdamW with cosine decay and Prodigy across all eight datasets on Qwen. AdamW with cosine decay remains the stronger of the two additional FO baselines, while Prodigy consistently trails both the scheduled AdamW baseline and the best ZFO method. However, the worst-performing ZFO variant falls below AdamW with cosine decay on several datasets, indicating that the advantage does not arise uniformly across every local model. Overall, these results suggest that ZFO can improve over many step-adaptive baselines, while also signaling that the choice of local approximation remains important.

\smartparagraph{ZFO Sequential Line Search.} The four ZFO methods in the main paper select a step using a local model constructed from a fixed budget of two probes. \texttt{ZFO-Seq} extends this design to a sequential budget of $k > 2$ fixed-batch evaluations. The first two probes construct the initial model $M \in \{T_2, T_3, P_2, P_3\}$ as in Algorithm~\ref{alg:zfo}, and the remaining $k-2$ probes are spent refining the model's proposal.

At each refinement step, the current candidate maximizer is evaluated on the shared minibatch $S_t$, and a new candidate is proposed by fitting a parabola through the best evaluated point in $I_t$ and its two nearest evaluated neighbors, aligning with safeguarded successive parabolic interpolation~\cite{brent}. The returned step is the best evaluated point rather than a model maximizer. Every candidate must survive an actual evaluation of $\phi_t$ before it can be selected. Since $(0,\phi_t(0)) \in E_t$ by construction, \texttt{ZFO-Seq} guarantees $\phi_t(s_t) \geq \phi_t(0)$ at every iteration, a monotone non-decrease of the empirical surrogate. This replaces the model-selection gap $\phi(s_\phi)-\phi(s_M)$ with the gap attained by the best of $k$ evaluated points, which shrinks under refinement. Algorithm~\ref{alg:zfo_seq} summarizes the procedure.

Table~\ref{tab:line_search_selected_results} reports the evaluation accuracy for a variety of line-search methods on the Qwen-2.5-Math-1.5B model (seed 100, configuration C, learning rate 1e-5, search bound $\beta=3$). Each ZFO-Seq model used $P_3$ as its default initialization.

\begin{algorithm}
\vspace{-1mm}
\caption{ZFO-Seq}
\label{alg:zfo_seq}
\begin{algorithmic}[1]
\Require  $\theta_0$, base optimizer, learning rate $\alpha$, $\epsilon>0$, $\beta>0$, budget $k\geq3$, initial model $M \in \{T_2,T_3,P_2,P_3\}$
\For{$t=0,1,2,\ldots$}
    \State Algorithm~\ref{alg:zfo} Inner Loop
    \For{$j=1,\ldots,k-2$}
        \State Evaluate $\phi_t(s^{(j)})$ on $S_t$ and add $(s^{(j)},\phi_t(s^{(j)}))$ to $E_t$
        \If {$j < k-2$}
        \State Fit a parabola $\mathcal {M}$ through the best point in $E_t$ and its two nearest neighbors
        \State Propose $s^{(j+1)} \in \arg\max_{s\in I_t} \mathcal{M}(s)$
        \EndIf
    \EndFor
    \State Select $s_t \in \arg\max_{s:\,(s,\phi_t(s))\in E_t} \phi_t(s)$
    \State Update $\theta_{t+1}=\theta_t+s_t z_t$
\EndFor
\end{algorithmic}
\end{algorithm}
\vspace{-2mm}

\begin{table}[t]
\centering
\caption{\textbf{Comparison with line-search algorithms.} Extended line-search results comparing the basic ZFO-Seq method against Armijo, Barzilai-Borwein, Polyak, and a simplified probabilistic line search (PLS) on the Qwen Model. The best result is in \textbf{bold}, and the second best result is \underline{underlined}.}
\label{tab:line_search_selected_results}
\begin{tabular}{lccccc}
\toprule
Dataset & ZFO-Seq & Armijo & BB & Polyak & PLS \\
\midrule
CODAH         & \textbf{55.47} & 32.81 & 52.73 & 30.47 & 31.25 \\
ARC-Challenge & \underline{55.86} & \textbf{57.81} & 54.69 & 35.16 & 35.16 \\
SVAMP         & \underline{83.59} & \textbf{85.55} & \textbf{85.55} & 82.81 & 76.17 \\
\bottomrule
\end{tabular}
\end{table}

Table~\ref{tab:line_search_selected_results} compares \texttt{ZFO-Seq} against several established line-search methods. \texttt{ZFO-Seq} achieves the strongest result on CODAH, while remaining within approximately two points of the best-performing method on both ARC-Challenge and SVAMP. Armijo and Barzilai-Borwein remain slightly stronger on these latter settings, whereas Polyak and PLS generally trail \texttt{ZFO-Seq}. Thus, \texttt{ZFO-Seq} is competitive with classical line-search and adaptive step-size methods across the three evaluated settings.

\smartparagraph{Unit-Bound Ablation.}
To determine whether ZFO benefits require a search interval larger than the standard first-order update, we additionally evaluate the four ZFO models with the bound fixed to $\beta=1$ on Qwen-2.5-Math-1.5B. Under this setting, ZFO may shorten or reshape the AdamW update but cannot select a step beyond the endpoint corresponding to the nominal FO displacement. ARC-Challenge and StrategyQA use three-epoch versions of configuration \emph{a}, while ASDiv uses a three-epoch version of configuration \emph{c}; we denote these shortened configurations by \emph{a'} and \emph{c'}, respectively. Their batch size, group size, context length, and maximum output length otherwise remain unchanged. All methods use the fixed RLVR settings in Table~\ref{tab:fixed_config}, including perturbation scale $\epsilon=10^{-3}$, and results are averaged over seeds $\{100,101,102\}$. Table~\ref{tab:beta1_llm_ablation} reports the resulting final evaluation accuracy.

\begin{table*}[t]
\centering
\small
\setlength{\tabcolsep}{5pt}
\renewcommand{\arraystretch}{0.95}
\caption{{\textbf{Qwen-2.5-Math-1.5B sub-study ablation with ${\beta=1}$.}
Results are reported as mean $\pm$ standard deviation over three seeds. All datasets are trained with 3 epochs (ARC-Challenge, StrategyQA; config a'), (ASDiv; config c'). 
Under this 3-epoch configuration, reducing the ARC FO learning rate from
$10^{-5}$ to $5\times10^{-6}$ improves accuracy from
$31.71\pm8.11$ to $55.12\pm0.60$. The ASDiv FO result uses the
third-epoch checkpoint of the matched FO runs.
}}
\vspace{-1mm}

\scalebox{0.96}{
\begin{tabular}{@{}lcccccc@{}}
\toprule
\textbf{Dataset}
& \textbf{Base LR}
& \textbf{AdamW (FO)}
& \textbf{Taylor$_2$}
& \textbf{Taylor$_3$}
& \textbf{Padé$_2$}
& \textbf{Padé$_3$} \\
\midrule

ARC-Challenge
& $1\times10^{-5}$
& \numcell{31.71 $\pm$ 8.11}
& \numcell{\textbf{56.77 $\pm$ 0.61}}
& \numcell{56.51 $\pm$ 0.36}
& \numcell{55.49 $\pm$ 0.98}
& \numcell{49.86 $\pm$ 0.27} \\

StrategyQA
& $1\times10^{-5}$
& \numcell{59.24 $\pm$ 10.14}
& \numcell{70.69 $\pm$ 5.10}
& \numcell{\textbf{74.87 $\pm$ 0.37}}
& \numcell{67.25 $\pm$ 13.64}
& \numcell{48.18 $\pm$ 0.44} \\

ASDiv
& $5\times10^{-6}$
& \numcell{\textbf{82.08 $\pm$ 0.08}}
& \numcell{81.82 $\pm$ 0.26}
& \numcell{78.05 $\pm$ 1.41}
& \numcell{81.87 $\pm$ 0.26}
& \numcell{81.92 $\pm$ 0.46} \\

\bottomrule
\end{tabular}
}
\label{tab:beta1_llm_ablation}
\vspace{-3mm}
\end{table*}

Table~\ref{tab:beta1_llm_ablation} shows that ZFO can modify training behavior even when it is not permitted to move beyond the nominal FO update. On ARC-Challenge, all four ZFO variants outperform AdamW at the shared learning rate of $10^{-5}$; however, reducing the separately tuned FO learning rate to $5\times10^{-6}$ raises its accuracy to $55.12$, so the ARC result is best interpreted as ZFO recovering a competitive step scale. On StrategyQA, \texttt{Taylor$_3$} performs best at $74.87$, followed by \texttt{Taylor$_2$} and \texttt{Pad\'e$_2$}, whereas \texttt{Pad\'e$_3$} falls below FO. On ASDiv, FO remains narrowly best, although \texttt{Taylor$_2$}, \texttt{Pad\'e$_2$}, and \texttt{Pad\'e$_3$} remain within $0.26$ points. Thus, $\beta=1$ is sufficient for gains in some settings, but we argue that the ZFO methods' largest gains arise from the ability to safely navigate larger learning rates.

\smartparagraph{Does ZFO require tuning the search bound?}
To examine performance without selecting $\beta$ separately for each
method, we fix $\beta=3$, permitting steps up to three times the nominal FO displacement. We retain the
corresponding main-experiment learning rates and training configurations. All results use seeds $\{100,101,102\}$.

\begin{table}[htbp]
\centering
\small
\setlength{\tabcolsep}{4pt}
\renewcommand{\arraystretch}{1.12}
\caption{Final evaluation accuracy (\%, mean $\pm$ standard deviation)
with $\beta=3$ for all ZFO variants. These selected settings illustrate
different local-model behaviors. Bold denotes the highest mean in each
row. Matching runs are reused from the main experiments; the
Phi-2/OpenBookQA Taylor$_2$ entry uses a separate three-seed repeat
of the same recorded configuration.}
\label{tab:fixed_beta_three}
\begin{tabular}{@{}lccccc@{}}
\toprule
Model / Dataset & FO & Taylor$_2$ & Taylor$_3$ & Pad\'e$_2$ & Pad\'e$_3$ \\
\midrule
Phi-2 / OpenBookQA
& $51.20 \pm 1.11$
& $\mathbf{58.87} \pm 12.96$
& $50.67 \pm 0.81$
& $49.40 \pm 0.20$
& $52.07 \pm 0.23$ \\
Gemma / SVAMP
& $52.00 \pm 1.76$
& $\mathbf{59.56} \pm 0.51$
& $47.22 \pm 2.01$
& $59.00 \pm 1.76$
& $58.22 \pm 3.69$ \\
Llama / GSM8K
& $4.85 \pm 5.47$
& $4.02 \pm 3.98$
& $\mathbf{9.00} \pm 6.68$
& $5.05 \pm 6.94$
& $1.72 \pm 0.88$ \\
Phi-2 / AsDiv
& $56.51 \pm 18.76$
& $65.42 \pm 2.59$
& $55.85 \pm 2.55$
& $62.05 \pm 7.38$
& $\mathbf{69.27} \pm 1.92$ \\
Qwen / SVAMP
& $88.11 \pm 0.69$
& $90.78 \pm 0.96$
& $86.33 \pm 0.33$
& $\mathbf{90.89} \pm 1.02$
& $90.33 \pm 2.19$ \\
\bottomrule
\end{tabular}
\end{table}

Table~\ref{tab:fixed_beta_three} shows that improvements over FO remain
possible with a fixed search bound. Each local model achieves the
highest mean in at least one selected setting: Taylor$_2$ on
Phi-2/OpenBookQA and Gemma/SVAMP, Taylor$_3$ on Llama/GSM8K,
Pad\'e$_2$ on Qwen/SVAMP, and Pad\'e$_3$ on Phi-2/AsDiv. As the smallest expanded bound in our sweep, $\beta=3$ provides
a conservative starting point relative to $\beta=5$ or $10$,
allowing exploration beyond the nominal FO step while limiting
the extent of extrapolation. Its suitability still depends on
the base learning rate, since the search radius scales with
$\alpha\beta$. Thus, $\beta=3$ offers a simple initial choice
whose effective search scale should be interpreted alongside
the task-specific learning-rate range.

\smartparagraph{Location of ZFO Steps.}
We next examine where the ZFO methods select steps within the search interval
$I_t=[0,R]$. This diagnostic distinguishes between
selecting no update ($s_t=0$), the largest permitted update ($s_t=R$), and an
interior step. Because numerical solutions may lie arbitrarily close to an
endpoint, we report both the exact interior frequency, $0<s_t<R$, and a
1\%-trimmed interior frequency, $0.01<s_t/R<0.99$. We also report the median
normalized step $s_t/R$. Statistics are computed within each run. Table~\ref{tab:step_selection_frequencies} summarizes the
$\beta=1$ Family-B cohort, consisting of 48 runs across four configurations,
and the fixed-$\beta=3$ cohort, consisting of 108 runs across nine
model-dataset pairs.

\providecommand{\numcell}[1]{{\fontsize{6.6}{7.2}\selectfont #1}}

\begin{table*}[t]
\centering
\small
\setlength{\tabcolsep}{5pt}
\renewcommand{\arraystretch}{0.95}
\caption{\small{\textbf{Distribution of ZFO-selected steps over the search
interval $[0,R]$.}
Frequencies are macro-averaged across runs. Exact and 1\%-trimmed interior
selections are reported separately, and median steps are normalized by $R$.
For \texttt{Pad\'e$_3$}, fallback frequency denotes activation of the
quadratic safeguard. The largest interior frequency for each value of
$\beta$ is bold.}}
\vspace{-1mm}

\scalebox{0.96}{%
\begin{tabular}{@{}llcccccc@{}}
\toprule
\multirow{2}{*}{$\boldsymbol{\beta}$} &
\multirow{2}{*}{\textbf{Method}} &
\multicolumn{4}{c}{\textbf{Selected-step frequency (\%)}} &
\multirow{2}{*}{\shortstack{\textbf{Median}\\$\boldsymbol{s_t/R}$}} &
\multirow{2}{*}{\shortstack{\textbf{Fallback}\\\textbf{frequency (\%)}}} \\
\cmidrule(lr){3-6}
& &
$s_t=0$ &
$0<s_t<R$ &
$0.01<s_t/R<0.99$ &
$s_t=R$ & & \\
\midrule

\multirow{4}{*}{$1$}
& \texttt{Taylor$_2$}
& \numcell{1.37}
& \numcell{40.97}
& \numcell{3.76}
& \numcell{57.66}
& \numcell{1.00000}
& \numcell{0.61} \\

& \texttt{Taylor$_3$}
& \numcell{0.86}
& \numcell{\textbf{65.83}}
& \numcell{\textbf{23.88}}
& \numcell{33.31}
& \numcell{0.00156}
& \numcell{0.67} \\

& \texttt{Pad\'e$_2$}
& \numcell{43.41}
& \numcell{0.00}
& \numcell{0.00}
& \numcell{56.59}
& \numcell{1.00000}
& \numcell{0.76} \\

& \texttt{Pad\'e$_3$}
& \numcell{76.27}
& \numcell{3.45}
& \numcell{2.55}
& \numcell{20.27}
& \numcell{0.00000}
& \numcell{12.49} \\

\midrule

\multirow{4}{*}{$3$}
& \texttt{Taylor$_2$}
& \numcell{5.04}
& \numcell{30.37}
& \numcell{15.87}
& \numcell{64.59}
& \numcell{1.00000}
& \numcell{$<0.01$} \\

& \texttt{Taylor$_3$}
& \numcell{4.10}
& \numcell{\textbf{80.77}}
& \numcell{\textbf{70.28}}
& \numcell{15.13}
& \numcell{0.02287}
& \numcell{0.01} \\

& \texttt{Pad\'e$_2$}
& \numcell{27.65}
& \numcell{0.00}
& \numcell{0.00}
& \numcell{72.35}
& \numcell{1.00000}
& \numcell{$<0.01$} \\

& \texttt{Pad\'e$_3$}
& \numcell{23.41}
& \numcell{12.23}
& \numcell{6.93}
& \numcell{64.36}
& \numcell{1.00000}
& \numcell{54.95} \\

\bottomrule
\end{tabular}%
}
\label{tab:step_selection_frequencies}
\vspace{-4mm}
\end{table*}

Table~\ref{tab:step_selection_frequencies} reveals distinct step-selection
behaviors across the four ZFO models. \texttt{Taylor$_3$} makes the most
frequent use of the interval interior under both bounds. Its exact interior
frequency increases from $65.83\%$ at $\beta=1$ to $80.77\%$ at $\beta=3$,
while its 1\%-trimmed interior frequency rises more substantially from
$23.88\%$ to $70.28\%$. The difference between the two interior measures at
$\beta=1$, together with a median normalized step of $0.00156$, indicates that
many of its nominally interior selections lie close to the zero endpoint. At
$\beta=3$, however, most selections remain interior even after trimming,
showing that \texttt{Taylor$_3$} frequently chooses steps that are meaningfully
separated from both boundaries.

The other models rely more heavily on endpoint selections.
\texttt{Taylor$_2$} most often selects the full step $s_t=R$, doing so in
$57.66\%$ of updates at $\beta=1$ and $64.59\%$ at $\beta=3$.
\texttt{Pad\'e$_2$} never selects an interior point in either cohort and
therefore acts empirically as a binary rule between $s_t=0$ and $s_t=R$.
\texttt{Pad\'e$_3$} also remains largely endpoint-driven. The median of \texttt{Pad\'e$_3$} shifts
from zero at $\beta=1$ to the upper endpoint at $\beta=3$. Its fallback
frequency simultaneously increases from $12.49\%$ to $54.95\%$, indicating
that the quadratic safeguard becomes substantially more important over the
larger search interval. Overall, the results show that the cubic Taylor model
uses higher-order information primarily to select intermediate step lengths,
whereas the remaining models more often use their local approximations to
choose between suppressing and taking the full permitted update.

\section{Notation}
\label{app:notation}

{\footnotesize
\begin{xltabular}{0.95\linewidth}{
  @{} p{0.28\linewidth} @{\hspace{1em}} X @{}}

\caption{Table of notations.}
\label{tab:notation}\\
\toprule
\textbf{Notation} & \textbf{Definition} \\
\midrule
\endfirsthead

\caption[]{Table of notations continued.}\\
\toprule
\textbf{Notation} & \textbf{Definition} \\
\midrule
\endhead

\midrule
\multicolumn{2}{r}{\emph{Continued on next page}}\\
\endfoot

\bottomrule
\endlastfoot

$\|\cdot\|$ & Euclidean norm.\\
\midrule
$\pi_\theta(y\mid x)$ & Conditional language model with parameters $\theta$, prompt $x$, and response $y$.\\
\midrule
$\theta_t$ & Model parameters at iteration $t$, updated as $\theta_{t+1}=\theta_t+s_t z_t$.\\
\midrule
$S_t$ & Sampled information held fixed during iteration $t$, defining the current surrogate and shared across its evaluations.\\
\midrule
$F_t(\theta)=F(\theta;S_t)$ & Fixed-batch surrogate objective optimized during iteration $t$.\\
\midrule
$\bar F_t(\theta)$ & Expected surrogate under the sampling distribution at iteration $t$.\\
\midrule
$J(\theta),J^\star$ & Global objective used in the stationarity analysis and its upper bound, $J(\theta)\leq J^\star$.\\
\midrule
$g_t$ & Update direction produced by the base first-order optimizer.\\
\midrule
$z_t$ & Normalized search direction, $z_t=g_t/\|g_t\|$ for $g_t\neq 0$.\\
\midrule
$\alpha$ & Base learning rate of the first-order optimizer.\\
\midrule
$\beta$ & Search-radius multiplier relative to the standard first-order displacement.\\
\midrule
$s_{\rm FO}$ & Length of the standard first-order step, $s_{\rm FO}=\alpha\|g_t\|$.\\
\midrule
$R_t,I_t$ & Search radius $R_t=\beta\alpha\|g_t\|$ and feasible step interval $I_t=[0,R_t]$ in the theoretical formulation.\\
\midrule
$\phi_t(s)$ & Directional fixed-batch surrogate, $\phi_t(s)=F_t(\theta_t+s z_t)$. The iteration index is omitted when unambiguous.\\
\midrule
$\epsilon$ & Physical probe radius for the symmetric evaluations $\phi_t(\epsilon)$ and $\phi_t(-\epsilon)$.\\
\midrule
$a_t$ & Directional slope supplied by the backward pass, $a_t=\phi_t'(0)=\langle\nabla F_t(\theta_t),z_t\rangle$.\\
\midrule
$b_t,c_t$ & Finite-difference estimates of the second and third directional derivatives, $b_t=\widehat{\phi_t''}(0)$ and $c_t=\widehat{\phi_t'''}(0)$.\\
\midrule
$T_2(s),T_3(s)$ & Quadratic and cubic Taylor models constructed from $\phi_t(0)$ and the available directional derivative estimates.\\
\midrule
$P_2(s),P_3(s)$ & Pad\'e $[1/1]$ and $[1/2]$ local models used by Pad\'e$_2$ and Pad\'e$_3$, respectively.\\
\midrule
$M(s)$ & Chosen local model, $M\in\{T_2,T_3,P_2,P_3\}$. Distinct from the scalar derivative bound $M$ below.\\
\midrule
$s_M$ & Model-maximizing step, $s_M\in\arg\max_{s\in I_t}M(s)$.\\
\midrule
$s_t$ & Step applied at iteration $t$, selected using the local model and any applicable safeguard.\\
\midrule
$s_\phi$ & Maximizer of the empirical directional surrogate, $s_\phi\in\arg\max_{s\in I_t}\phi_t(s)$.\\
\midrule
$\tau_i,n$ & Shared sampled example $\tau_i$ and number of examples $n$ in the CRN analysis.\\
\midrule
$\phi_i(s)$ & Per-sample directional objective, $\phi_i(s)=F(\theta+s z;\tau_i)$, with $\phi(s)=\frac{1}{n}\sum_{i=1}^{n}\phi_i(s)$.\\
\midrule
$\psi(s)$ & Population directional objective, $\psi(s)=\mathbb{E}_{\tau}[\phi_i(s)]$; at iteration $t$, $\psi(s)=\bar F_t(\theta_t+s z_t)$.\\
\midrule
$W,M,B,L$ & Uniform bounds on $|\phi_i''(s)|$, $|\phi_i'''(s)|$, $|\phi_i^{(4)}(s)|$, and $|\phi_i^{(5)}(s)|$, respectively, for all $i$ and $s\in[-\epsilon,R]$.\\
\midrule
$\sigma$ & Sub-Gaussian parameter controlling deviations of the per-sample directional derivatives from their population counterparts in Assumption~4.\\
\midrule
$\delta$ & Failure-probability parameter in the concentration and step-selection bounds.\\
\midrule
$N(s),D(s)$ & Numerator and denominator of a Pad\'e model, $P_k(s)=N(s)/D(s)$.\\
\midrule
$r(s)$ & Pad\'e residual, $r(s)=D(s)\phi(s)-N(s)$.\\
\midrule
$\eta$ & Positive lower bound on the denominator magnitude, $|D(s)|\geq\eta>0$ for $s\in[0,R]$.\\
\midrule
$Q,P$ & Uniform residual derivative bounds: $|r^{(3)}(s)|\leq Q$ for Pad\'e$_2$ and $|r^{(4)}(s)|\leq P$ for Pad\'e$_3$ on $[0,R]$.\\
\midrule
$\mathcal{L}_J,\mathcal{L}_{\bar F}$ & Lipschitz constants of the gradients of $J$ and the expected surrogate $\bar F$, respectively.\\
\midrule
$\mathbb{E}_t$ & Conditional expectation given the history before iteration $t$.\\
\midrule
$q_t$ & Expected directional alignment, $q_t=\mathbb{E}_t[\langle\nabla J(\theta_t),z_t\rangle]$.\\
\midrule
$\mu$ & Positive alignment constant satisfying $q_t\geq\mu\|\nabla J(\theta_t)\|$.\\
\midrule
$\Delta_t$ & Error bound from Lemma~1 on the loss in $J$ relative to the best step in $I_t$, holding with the stated conditional probability.\\
\midrule
$\widetilde{\Delta}_T$ & Average expected step-selection error over $T$ iterations, $\widetilde{\Delta}_T=\frac{1}{T}\sum_{t=0}^{T-1}\mathbb{E}[\Delta_t]$.\\
\midrule
$K$ & Combined smoothness constant, $K:=\mathcal{L}_J+\mathcal{L}_{\bar F}+2W$.\\
\midrule
$\chi$ & Positive coefficient in Theorem~\ref{thm:zfo-stationarity}, $\chi:=3\mu^2-4\delta>0$.\\

\end{xltabular}
}

\smartparagraph{Impact Statement.} This work improves step-size selection for first-order optimization, making fine-tuning more stable and less sensitive to learning-rate tuning, which can reduce the compute spent on hyperparameter search. As a general-purpose optimization method, ZFO does not introduce application-specific risks beyond those already associated with training and fine-tuning machine learning models.

%% file: Sections/02-related-work.tex
\tableofcontents

\smartparagraph{Organization.} We organize the Appendix as follows: In Section \ref{sec:related_work}, we detail the related work. In Section \ref{sec:add_theory}, we quote the general assumptions and provide detailed proofs of our theoretical results in the main paper. Finally, in Section \ref{sec:perf_insights}, we discuss the configuration details for our experiments in \S\ref{sec:configuration} and provide additional performance insights in \S\ref{sec:add_insights}.

\section{Related work}\label{sec:related_work}

We position ZFO within three lines of work that motivate or contrast with our approach. Together, these related works inspire ZFO's central premise that cheap query calls can help model local objective behavior along the gradient direction. 

\smartparagraph{Policy Optimization Stability \& Update Regulation.} Stability in policy optimization has motivated a range of update-regulation strategies. In policy optimization, methods such as TRPO~\cite{schulman2015trust} enforce a KL-constrained trust region, while PPO~\cite{schulman2017proximal} replaces explicit constraints with a clipped surrogate objective to stabilize updates. More recent approaches, such as GRPO~\cite{grpo}, further emphasize that training stability is highly sensitive to the way policy updates are regulated. Related works also consider wrapper-based stabilization strategies for policy-gradient methods that regulate whether an update should be accepted~\cite{ahmed2026stabilizing}. Our work is motivated by the same stability concern; however, we approach it from a different perspective by explicitly modeling the objective along the update direction path.

\smartparagraph{Zeroth-Order Optimization.} Our work is also related to ZO optimization for fine-tuning large language models. Methods such as MeZO and CARS use central finite differences along random directions to approximate gradient or curvature behavior~\cite{mezo, cars}; particularly, MeZO uses finite differences along random directions to approximate the gradient, and CARS uses a second-order finite difference to approximate the Hessian. Extensions of MeZO include curvature, sparsity, and query-efficient estimators. Hi-ZOO incorporates a diagonal Hessian estimate at the cost of an extra forward pass~\cite{zhaosecond}. Sparse-MeZO restricts updates to noise-resistant weights to enable larger step sizes~\cite{liusparse}. GraCe leverages gradient sparsity to achieve a dependence on dimension in query complexity~\cite{qiu2024gradient}. HiSo carries Hessian-informed ZO into the federated setting~\cite{li2026converge}. Complementary benchmarking has mapped the landscape of ZO optimizers across fine-tuning paradigms~\cite{zhang2024revisiting}. Recent theoretical work further establishes that ZO optimization is equivalent to single-step policy optimization~\cite{qiu2025zeroth}, suggesting a natural connection between zeroth-order methods and RLVR that ZFO exploits.

\smartparagraph{Line-Search \& Adaptive Step Size.} Broadly, our method is similar to line-search and adaptive learning-rate methods, including Polyak step size~\cite{polyak}, Barzilai-Borwein~\cite{barzi}, and stochastic variants such as probabilistic line-search~\cite{mahsereci2017probabilistic} and stochastic Armijo line-search~\cite{stoch}. These approaches adapt step sizes by repeatedly evaluating candidate points along a univariate objective, often requiring multiple queries. In contrast, our method uses only two symmetric probes to approximately recover higher-order directional information to construct a local model of the objective. This results in a fixed-budget alternative to classical line-search methods.

\smartparagraph{Generalized Newton Step Selection.} Among closely related recent methods, generalized Newton step selection or GeN~\cite{bu2025gradient} has already explored function-based step selection along a first-order direction, providing an early example of combining FO directions with low-query step-size adaptation. Our contribution is to elevate this idea from a single-step rule to a fixed-query directional modeling framework: the same two probes can support multiple local models, including higher-order Taylor and rational Padé approximations, while the analysis separates finite-difference estimation error, local-model approximation error, and global optimization error. GeN is a special case of the broader ZFO abstraction, rather than an unrelated baseline.